\documentclass[10pt,english]{article}
\usepackage[margin=1in]{geometry} 
\usepackage[T1]{fontenc}
\usepackage[latin9]{inputenc}
\usepackage{bm}
\usepackage{amsmath}
\usepackage{amssymb} 
\usepackage{tabularx}
\usepackage{pifont}
\usepackage{tablefootnote}
\usepackage{longtable}
\usepackage{lipsum} 
\usepackage{caption}
\usepackage{bbding}
\usepackage{makecell}
\usepackage{threeparttable}
\usepackage[labelformat=simple]{subcaption}

\usepackage[unicode=true,
 bookmarks=false, 
 breaklinks=false,pdfborder={0 0 1},colorlinks=false]
 {hyperref}
\hypersetup{
 colorlinks,citecolor=blue,filecolor=blue,linkcolor=blue,urlcolor=blue}
\usepackage{xcolor,colortbl}
\definecolor{Gray}{gray}{0.85}
\usepackage{enumitem}
\makeatletter
\usepackage{amsthm}  
\usepackage{comment}
\usepackage{booktabs,mathtools}
\usepackage{graphicx}
\usepackage{algorithm}
\usepackage{algpseudocode}
\usepackage{comment}
\newcommand{\mycomment}[1]{}
\usepackage{thmtools}
\usepackage{thm-restate}

\usepackage{multirow}
\usepackage{dsfont}
\usepackage{color}
\usepackage{float}

\definecolor{full}{RGB}{225,0,100}
\definecolor{hew}{RGB}{0,47,167}

\definecolor{own_pink}{RGB}{217,25,169}
\definecolor{own_blue}{RGB}{0,100,223}

\allowdisplaybreaks

\newcommand{\cf}{cf.}
\newcommand{\ie}{i.e.}
\newcommand{\eg}{e.g.}
\newcommand{\alg}{\texttt{PISCO}}

\newcommand{\bx}{\boldsymbol{x}}
\newcommand{\bv}{\boldsymbol{v}}
\newcommand{\bxi}{\boldsymbol{\xi}}
\newcommand{\by}{\boldsymbol{y}}
\newcommand{\bz}{\boldsymbol{z}}
\newcommand{\bX}{\boldsymbol{X}}
\newcommand{\bY}{\boldsymbol{Y}}

\newcommand{\bG}{\boldsymbol{G}}
\newcommand{\bg}{\boldsymbol{g}}

\newcommand{\bW}{\boldsymbol{W}}
\newcommand{\bJ}{\boldsymbol{J}}

\newcommand{\bA}{\boldsymbol{A}}
\newcommand{\bB}{\boldsymbol{B}}
\newcommand{\bC}{\boldsymbol{C}}
\newcommand{\be}{\boldsymbol{e}}
\newcommand{\bI}{\boldsymbol{I}}
\newcommand{\bPhi}{\boldsymbol{\Phi}}

\newcommand{\bDelta}{\boldsymbol{\Delta}}

\newcommand{\barx}{\overline{\boldsymbol{x}}}

\newcommand{\barg}{\overline{\boldsymbol{g}}}
\newcommand{\barX}{\overline{\boldsymbol{X}}}
\newcommand{\barY}{\overline{\boldsymbol{Y}}}
\newcommand{\barG}{\overline{\boldsymbol{G}}}

\newcommand{\E}{\mathbb{E}}
\newcommand{\F}{\mathsf{F}}

\newcommand{\calZ}{\mathcal{Z}}
\newcommand{\calD}{\mathcal{D}}

\newcommand{\normF}{\mathsf{F}}

\theoremstyle{plain} 
\newtheorem{lemma}{\textbf{Lemma}}
\newtheorem{corollary}{\textbf{Corollary}}
\newtheorem{proposition}{\textbf{Proposition}}
\newtheorem{definition}{\textbf{Definition}}
\newtheorem{thm}{\textbf{Theorem}} 
\newtheorem{assump}{Assumption}
\theoremstyle{remark}\newtheorem{remark}{\textbf{Remark}}

\title{Communication-Efficient Federated Optimization\\ over Semi-Decentralized Networks}

 \author{
 	He Wang\thanks{Department of Electrical and Computer Engineering, Carnegie Mellon University, Pittsburgh, PA 15213, USA.}\\
 	Carnegie Mellon University\\
 	\texttt{hew2@andrew.cmu.edu}
 	\and
 	Yuejie Chi\footnotemark[1] \\ 	 
  	Carnegie Mellon University  \\
 	\texttt{yuejiechi@cmu.edu}
	\thanks{A preliminary version of this work was presented at 2024 IEEE International Conference on Acoustics, Speech and Signal Processing (ICASSP). Initial submission: September 2023.}
 	}

\date{September 2023; Revised \today}

\begin{document}

\maketitle

 \begin{abstract} 
	In large-scale federated and decentralized learning, communication efficiency is one of the most challenging bottlenecks. While gossip communication---where agents can exchange information with their connected neighbors---is more cost-effective than communicating with the remote server, it often requires a greater number of communication rounds, especially for large and sparse networks. To tackle the trade-off, we examine the communication efficiency under a \textit{semi-decentralized} communication protocol, in which agents can perform both agent-to-agent and agent-to-server communication \textit{in a probabilistic manner}. We design a tailored communication-efficient algorithm over semi-decentralized networks, referred to as \alg, which inherits the robustness to data heterogeneity thanks to gradient tracking and allows multiple local updates for saving communication. We establish the convergence rate of \alg~for nonconvex problems and show that \alg~enjoys a linear speedup in terms of the number of agents and local updates. Our numerical results highlight the superior communication efficiency of \alg~and its resilience to data heterogeneity and various network topologies.
 \end{abstract}


\allowdisplaybreaks
\setcounter{tocdepth}{2}
\tableofcontents

\section{Introduction}

Consider a networked system that $n$ agents collectively solve the following federated or distributed optimization problem:
\begin{equation}\label{eq:prob}
  \min _{\boldsymbol{x} \in \mathbb{R}^d} f(\boldsymbol{x})\coloneqq \frac{1}{n} \sum_{i=1}^n f_i(\boldsymbol{x}), \; \text{where}~f_i(\boldsymbol{x}) \coloneqq \frac{1}{m}\sum_{\boldsymbol{z}\in \mathcal{D}_i} \ell(\boldsymbol{x};\boldsymbol{z}).
\end{equation}
Here, $\boldsymbol{x}\in\mathbb{R}^d$ denotes the optimization variable, $f_i(\boldsymbol{x})$ denotes the local and private objective function at agent $i$, and $f(\boldsymbol{x})$ denotes the global objective function. In addition,  let $\boldsymbol{z}$ represent one data sample, $\mathcal{D}_i$ stand for the dataset with $|\mathcal{D}_i|=m$ samples at agent $i$, and $\ell(\boldsymbol{x};\boldsymbol{z})$ denote the empirical loss of $\boldsymbol{x}$ w.r.t. the data sample $\boldsymbol{z}$. Such problems have a wide range of applications, including but not limited to estimation in sensor networks~\cite{rabbat2004sensor}, resource allocation in smart grids~\cite{beck2014resource}, and coordination in multi-agent systems~\cite{cao2013coordiation}.

In order to tackle this problem, agents have to communicate with one another for cooperation, since every agent $i\in [n]$ only has access to its own local dataset $\mathcal{D}_i$. There are two main communication protocols, consisting of agent-to-agent communication model (in decentralized ML) and agent-to-server communication model (in federated ML). Commonly, they are formulated via different network topologies~\cite{nedich2018networktopology}, as shown in Figure \ref{fig:communication_models}.
More specifically, prior works in decentralized ML often use a general graph to capture the local communication, where every agent is only allowed to exchange information with its connected neighbors (\cf~Figure~\ref{fig:communication_models_decentralized}). In federated ML, the star graph is commonly used to depict the communication between agents and the centralized coordinator (\ie, server) who can both collect information from and broadcast to each agent (\cf~Figure~\ref{fig:communication_models_federated}). 

\begin{figure}[!htbp]
    \centering 
    \begin{subfigure}[b]{0.45\textwidth}
    \centering
    \includegraphics[width = 0.6\linewidth,trim={7cm 21cm 44cm 15cm},clip]{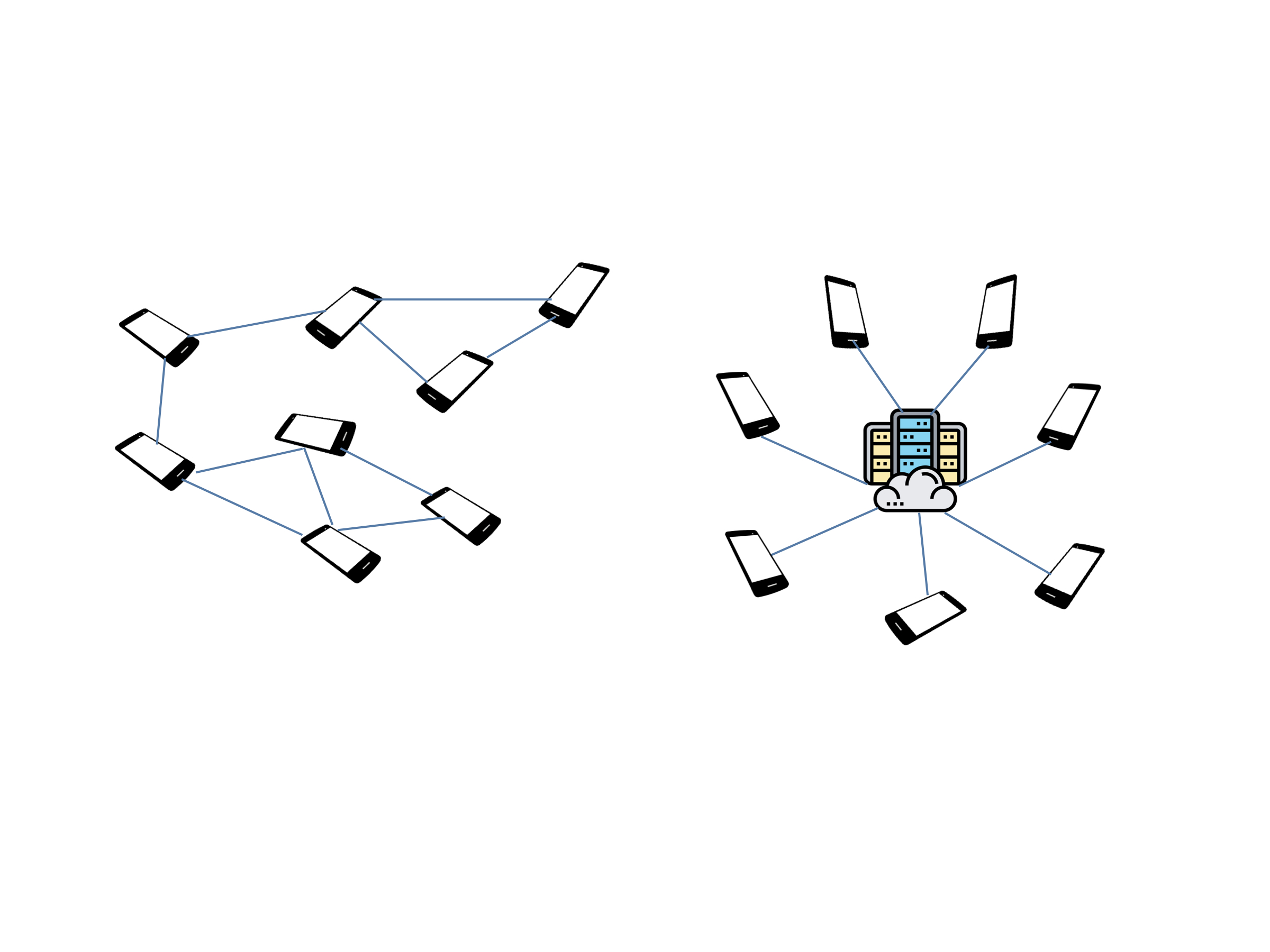}
    \caption{Agent-to-agent communication model with a general connected graph, where each agent only communicates with its adjacent agents. }
    \label{fig:communication_models_decentralized}
    \end{subfigure}
    \qquad
    \begin{subfigure}[b]{0.45\textwidth}
    \centering
    \includegraphics[width = 0.6\linewidth,trim={43cm 19cm 10cm 16cm},clip]{fig/communication_models.jpeg}
    \caption{Agent-to-server communication model with a star graph, where each agent can send messages to and receive messages from the server.}
    \label{fig:communication_models_federated}
    \end{subfigure}
    \caption{Two communication models for distributed ML.}
    \label{fig:communication_models}
\end{figure}

As the network size increasingly grows, communication efficiency becomes so critical that significantly hinders both decentralized and federated ML from being applied to real-world applications. Compared with agent-to-server communication, agent-to-agent communication is much more affordable and more applicable to large-scale networks. However, without the coordination of the server, decentralized approaches may need more communication rounds to reach consensus, especially for large and sparse networks.

Given that the communication complexity depends on the trade-off between the communication rounds and the per-round cost, emerging works focus on heterogeneous communication over \textit{semi-decentralized networks}, to gain the best from both agent-to-agent and agent-to-server communication \cite{chen2021accelerating,wang2022accelerating,guo2022hybrid}. Such semi-decentralized networks---consisting of a centralized server and a network of agents---widely exist in many applications, such as autonomous vehicles \cite{carvajal2021autovehicle}, energy systems \cite{navidi2021semi} and ML systems \cite{miao2023paralleltraining}. It has been observed that heterogeneous communication largely alleviates the heavy network dependence of distributed learning and tackles the communication bottleneck of the server \cite{chen2021accelerating}. However, to the best of our knowledge, all of them rely on the assumption of bounded data dissimilarity across agents and a complete characterization of the convergence behavior with respect to the network heterogeneity is still lacking. More detailed discussions on communication-saving strategies and semi-decentralized approaches are provided in Section \ref{sec:related_work}.

\begin{figure}[t]
    \centering
    \includegraphics[width = 0.6\linewidth,trim={2.5cm 15cm 4cm 9cm},clip]{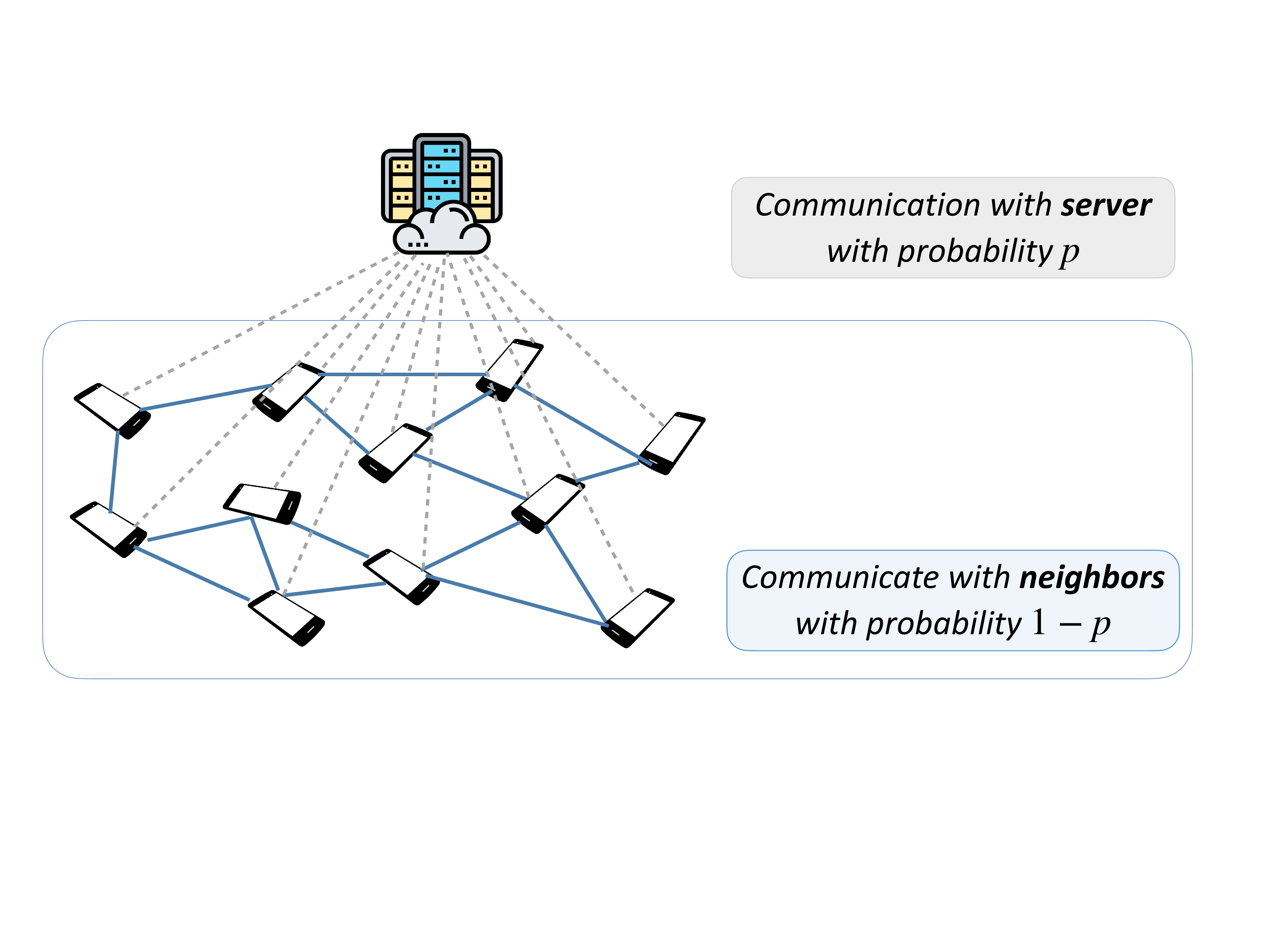}
    \caption{The semi-decentralized communication protocol, where the server can be accessed with probability $p$ and agents can communicate with their neighbors whenever the server is not available. Here, dotted lines represents the agent-to-server communication, while the solid ones are for agent-to-agent communication.}
    \label{fig:semidecentralized}
\end{figure}

\subsection{Our contributions}

To fill the void, we propose a  communication-efficient algorithm called \alg, which incorporates \textit{gradient-tracking techniques} \cite{zhu2010discrete} and \textit{multiple local updates} \cite{mcmahan2017communication} for solving federated nonconvex optimization over semi-decentralized networks modeled by a probabilistic connection model (shown in Figure \ref{fig:semidecentralized}). Such a semi-decentralized communication model (with local updates) allows \alg~to be viewed as a special form of gradient-tracking-based algorithms with time-varying networks. However, existing convergence guarantees for nonconvex optimization \cite{lorenzo2016next,lu2020decentralized,huang2022optimal}, cannot fully characterize the benefits of agent-to-server communication and multiple local updates. Specifically, applying previous analyses would yield convergence results that depend on the spectral gap of the least connected network---\ie, the underlying gossip communication network---while failing to capture the value of agent-to-server communication. To quantify these benefits, our analysis is of independent interest and can be readily extended to time-varying networks.
The highlights of our contributions are as follows.
\begin{enumerate}
    \item We prove that \alg~converges at a rate of  $O\left(1/\sqrt{nT_oK}\right)$ for sufficiently large $K$, where $K$ is the number of communication rounds and $T_o$ is the number of local updates. Our result does not impose the strong assumptions on data heterogeneity. Moreover, increasing the number of local updates accelerates the convergence over semi-decentralized networks. See Table \ref{tab:semi-ML} for a detailed comparison with prior art.
    
    \item We show that the communication heterogeneity offered by a semi-decentralized network  largely alleviates the network dependency of communication overheads in decentralized networks via a few agent-to-server communication rounds. For large and sparse networks (\ie, the mixing rate $\lambda_w\to 0$), with a small probability $p=\Theta\left(\sqrt{\lambda_w}\right)$ of agent-to-server communication, the network dependency improves from $O(\lambda_w^{-2})$ to $O(\lambda_w^{-1})$. 
    \item We corroborate the superior communication efficiency of \alg~through simulations on real-world datasets. The results substantiate the convergence speedup brought by multiple local updates and the robustness of  \alg~to data heterogeneity and various topologies, even for locally disconnected networks.
    
\end{enumerate}

\begin{table*}[!t]
    \centering
    \renewcommand{\arraystretch}{2}
     \begin{tabular}{|c |c|c| c| c| c|} 
     \hline
     \multirow{2}{*}{Algorithm} & \multicolumn{2}{c|}{Convergence rate \footnotemark[1]}& \multirow{2}{*}{\makecell[ct]{Bounded data\\heterogeneity\\assumption}} & \multirow{2}{*}{\makecell[ct]{Accessibility\\of server}} & \multirow{2}{*}{\makecell[ct]{Multiple\\local updates}}\\
     \cline{2-3}
     & Mini-batch & Large-batch& & & \\
     \hline
     \makecell{Gossip-PGA\\\cite{chen2021accelerating}} & $O\left(\frac{1}{\sqrt{nK}}\right)$ & $O\left(\left(\frac{G}{K}\right)^{\frac{2}{3}}\right)$ &   \ding{51}  & \makecell{every $H$ rounds} &\ding{55} \\
     \hline
     \makecell{HL-SGD\\\cite{guo2022hybrid}} &$O\left(\frac{1}{\sqrt{nK}}\right)$& $O\left(\left(\frac{G_1+G_2}{K}\right)^{\frac{2}{3}}\right)$  & \ding{51} &\makecell{every $H$ rounds}  &\ding{55}\\
     \hline
     \rowcolor{Gray} 
     \textbf{This paper} &$O\left(\frac{1}{\sqrt{nT_oK}}\right)$ & $O\left(\frac{1}{K}\right)$   & \ding{55} & w.p. $p$ &\ding{51}\\
     \hline
     \end{tabular}
     \caption{Comparison of ours and other semi-decentralized algorithms in nonconvex optimization using the same batch size, regarding the convergence rate, algorithm design and data heterogeneity assumptions. Here, $K$ is the number of communication rounds, $n$ is the number of agents, $T_o$ is the number of multiple local updates within a single gossip/global communication round, $G$ is the quantity in bounded dissimilarity between local objective and the global objective, $G_1$ is the quantity in bounded intra-cluster dissimilarity as \cite[Assumption 4]{guo2022hybrid}, $G_2$ is the quantity in bounded inter-cluster dissimilarity as \cite[Assumption 5]{guo2022hybrid}. }
     \label{tab:semi-ML}
\end{table*}
\footnotetext[1]{Here we only present the leading term of the rate for simplicity.}

\subsection{Related works}\label{sec:related_work}
Over the past few years, distributed optimization has attracted growing attention and has been extensively explored. For the convenience of our readers, we provide a review of the most related works below.
 
\paragraph{Distributed nonconvex optimization.}
As the size of the networked system increases, there are considerable algorithms developed for solving distributed nonconvex optimization. Roughly speaking, they can be categorized into two classes --- decentralized algorithms \cite{nedic09dgd,lorenzo2016next,hong2017prox_pda,tang2018d2} where agents are only allowed to exchange information with neighbors, and federated algorithms \cite{Li2020On,karimireddy2020scaffold,wang22fedadaptive} where agents are able to communicate with the server directly. Early attempts apply (stochastic) gradient descent to distributed optimization, which performs well in practice \cite{nedic09dgd,Li2020On}. However, the dissimilarity among local objectives could degenerate the performance under heterogeneous data and thus requires additional assumptions like bounded gradient or diminishing step-sizes. To eliminate such strong assumptions, many following works \cite{lorenzo2016next,hong2017prox_pda,karimireddy2020scaffold} have been developed, including gradient tracking (GT) techniques.
The key idea of GT is to utilize dynamic average consensus \cite{zhu2010discrete} for global gradient estimation, which has been incorporated with many distributed optimization algorithms to achieve faster convergence rates in nonconvex settings~\cite{scutari2019distributed,sun2020improving,sun2022inference,zhao2022beer}. 
Our proposed \alg~also takes advantage of GT to inherit its robustness against data heterogeneity.

\paragraph{Communication-efficient distributed ML.}
 Communication efficiency is one of the most important bottlenecks in distributed ML. In the decentralized setting, the communication complexity largely depends on the network topology, \ie, poor connectivity slows down the information mixing and thus requires more communication rounds to consensus \cite{nedich2018networktopology}. As for the federated ML, the communication burden of the centralized server may be unaffordable. 
To overcome such bottlenecks, a number of strategies are proposed for improving communication efficiency~\cite{cao2023review}, including: 1) \textit{compression methods}: compressing the information for communication~\cite{Koloskova*2020Decentralized,zhao2022beer}; 2) \textit{multiple local communication and updates}: executing multiple gossip communication~\cite{li2020communication,chen2021accelerating} or successive local updates within one communication round~\cite{nguyen2022performance,ge2023gradient,Li2020On,karimireddy2020scaffold}; 3) \textit{heterogeneous communication model over semi-decentralized networks}: allowing both agent-to-agent and agent-to-server communication \cite{chen2021accelerating,sun2023semi,lin2021semi-decentralized}; and \textit{4) adaptive communication strategies}: utilizing event-triggered communication mechanisms \cite{george2020distributed,wen2018distributed,he2023asymptotic} and tailored adaptive communication topologies \cite{song2022communication,ding2023dsgd,tupitsa2024federated,liu2024efficient} for saving unnecessary communications.
In this paper, we aim to gain the best communication efficiency from probabilistic heterogeneous communication over semi-decentralized networks and enable multiple local updates for more communication savings.

\paragraph{Semi-decentralized ML.}
As mentioned above, semi-decentralized ML, resorting to heterogeneous communication, tackles both the communication bottleneck of the centralized server in federated ML and the heavy network dependency of decentralized ML. 
We mainly summarize the works in nonconvex setting that are mostly related to this paper, while referring readers to \cite{chen2021accelerating,parasnis2023connectivity} for the (strongly) convex setting.
For nonconvex problems, Gossip-PGA \cite{chen2021accelerating} first integrates Gossip SGD \cite{nedic09dgd} with periodical global averaging and obtains a better scalability. It shows that intermittently communicating with the server can largely alleviate the heavy dependence on the network connectivity, especially for large or sparse networks. However, the theoretical results depend on the assumption of bounded similarity between local objectives.
Moreover, HL-SGD~\cite{guo2022hybrid} extends Gossip-PGA to the hierarchical networked structure with multiple clusters, while HA-Fed \cite{wang2022accelerating} can be viewed as HL-SGD with momentum.
Both of them enable intra-cluster gossip averaging and inter-cluster averaging, but also rely on the data heterogeneity assumptions which may be impractical in many real-world applications. 
Noted that all of them consider deterministic heterogeneous communication, \ie, agents/cluster can only communicate with the server every $H$ communication rounds, but the synchronization largely depends on the availability of the server. To this end, we consider the probabilistic communication model, where agents only exchange information with the server at the probability $p$. Furthermore, to the best of our knowledge, none of these approaches enable multiple local updates within a single communication round, whereas ours benefits from the linear speedup provided by the local updates. More detailed comparison can be found in Table \ref{tab:semi-ML}. Note that this comparison is based on the same batch size, while \cite{chen2021accelerating} and \cite{guo2022hybrid} may use fewer batch data per round.

%%%%%%%%%%%%%%%%%%%%%%%%%%%%%%%%%%%%%%%%%%%%
\subsection{Notation}
Throughout this paper, we use the lowercase and uppercase boldface letters to represent vectors and matrices, respectively. 
We use $\|\boldsymbol{A}\|_\normF$ for the Frobenius norm of a matrix $\boldsymbol{A}$, $\|\boldsymbol{A}\|_2$ for the largest singular value of a matrix $\boldsymbol{A}$, $\|\boldsymbol{a}\|_2$ for the $l_2$ norm of a vector $\boldsymbol{a}$, and $\otimes$ for the Kronecker product. In addition, we use $\boldsymbol{I}_n$ for the identity matrix of dimension $n$, $\boldsymbol{1}_n$ for the all-one vector of dimension $n$ and $\boldsymbol{O}_{d\times n}$ for the all-zero matrix of dimension $(d\times n)$. For any two real functions $f(\cdot)$ and $g(\cdot)$ defined on $\mathbb{R}^{+}$, $f(x) = O(g(x))$ if there exists a positive real constant $M$ and $x_0$ such that $f(x)\le M g(x)$ for any $x\ge x_0$. Similarly,  $f(x) = \Theta(g(x))$ if there exist positive real constants $M_1,M_2$ and $x_0$ such that $M_1 g(x)\le f(x)\le M_2 g(x)$ for any $x\ge x_0$.
Note that ``$\le$'' can be interpreted in an  element-wise fashion, if it is applied to vectors or matrices with the same dimension.

\section{Preliminaries}\label{sec:preliminaries}

\subsection{Communication graph and mixing matrix}
Consider a \textit{semi-decentralized network} that has a centralized server to coordinate all $n$ agents and an undirected \textit{communication graph} $\mathcal{G} = (\mathcal{V},\mathcal{E})$, where $\mathcal{V} = \{1,\ldots,n\}$ denotes the set of $n$ agents and $\mathcal{E} \subseteq \{\{i,j\}|i,j\in\mathcal{V}\}$ represents the local communication between agents. For every agent $i\in\mathcal{V}$, let $\mathcal{N}_i = \{j|\{i,j\}\in\mathcal{E}\}$ denote agent $i$'s neighbors whom the agent $i$ can communicate with. 

Moreover, for any communication graph $\mathcal{G}$, the mixing of local communication can be formally characterized by the \textit{mixing matrix} $\boldsymbol{W} = [w_{ij}]_{1\le i,j\le n}$ defined in Definition \ref{def:mixing_matrix}.
\begin{definition}[Mixing matrix and mixing rate]\label{def:mixing_matrix}
    Given an undirected communication graph $\mathcal{G}$, a nonnegative matrix $\boldsymbol{W}\in\mathbb{R}^{n\times n}$ is the mixing matrix, whose element $w_{ij}=0$ if and only if $\{i,j\}\notin \mathcal{E}$ and $i\ne j$ and $\boldsymbol{W}$ is doubly stochastic, \ie, $\boldsymbol{W1}_n = \mathbf{1}_n ~and~\mathbf{1}^{\top}_n\boldsymbol{W}=\mathbf{1}_n^{\top}$.
     The mixing rate of $\bW$ is a nonnegative constant, \ie,
    \begin{equation*}
        \lambda_w \coloneqq 1-\|\bW-\frac{1}{n}\mathbf{1}_n\mathbf{1}_n^{\top}\|_2^2.
    \end{equation*}
\end{definition}
Note that the doubly stochasticity implies $\|\bW\|_2\le 1$ and the mixing rate $\lambda_w = 1-\lambda^2 \in[0,1]$, where $\lambda$ denotes the second largest eigenvalue. The mixing rate can depict the connectivity of the communication graph $\mathcal{G}$, or to say, the speed of information mixing. Mathematically, 
\begin{equation*}
    \|\bW\bx-\barx\|_2^2\le  (1-\lambda_w)\|\bx-\barx\|_2^2, \quad \forall \bx\in\mathbb{R}^d,
\end{equation*}
where $\barx = \frac{1}{n}\mathbf{1}_n\mathbf{1}_n^{\top}\bx\in\mathbb{R}^{n}$. In other words, a larger mixing rate $\lambda_w$ indicates a better connectivity as well as a faster process of information mixing, while disconnected graphs have $\lambda_w=0$.  For example,  considering a fully connected graph where every agent can communicate with each other, the mixing matrix can be defined as $$\bJ \coloneqq \frac{1}{n}\mathbf{1}_n\mathbf{1}_n^{\top},$$ and its mixing rate is equal to $1$. Specifically, in this paper, we use $\bJ$ to describe the mixing of the agent-to-server communication.
\subsection{Stochastic gradient methods}
To improve the computational efficiency, one popular approach is to replace the full-batch gradients with stochastic gradients on the mini-batch data samples. In distributed setting, we define the local stochastic gradient for each agent $i\in[n]$ as:
\begin{equation}\label{eq:stochastic_grad}
    \boldsymbol{g}_i = \frac{1}{b}\sum_{\boldsymbol{z}_i\in\mathcal{Z}_i} \nabla \ell(\boldsymbol{x_i};\boldsymbol{z}_i), \quad\forall \bx_i\in\mathbb{R}^d,
\end{equation}
where $\mathcal{Z}_i\subset\mathcal{D}_i$ denotes the sampled data batch. Here, we assume that $\mathcal{Z}_i$ is drawn i.i.d. from $\mathcal{D}_i$  with the same mini-batch size $b<m$ for every agent $i\in [n]$ for simplicity, while using an adaptive batch size could be of interest for better controlling the variance of stochastic gradients \cite{bollapragada2018adaptive}. Note that the local stochastic gradient $\bg_i$ is an unbiased estimate of $\nabla f_i(\bx_i)$, \ie, 
\begin{equation*}
    \E [\bg_i] = \nabla f_i(\bx_i), \quad\forall \bx_i\in\mathbb{R}^d.
\end{equation*}
\subsection{Gradient-tracking techniques}
In many real-world applications, the local dataset $\mathcal{D}_i$ on every agent $i\in[n]$ may be quite different from each other, referred to as \textit{data heterogeneity}. Accordingly, there exists some local stationary solution $\bx$ satisfying $\nabla f_i(\bx) = 0$ for some $i\in[n]$, but not necessarily with $\nabla f(\bx) = \sum_{i=1}^n \nabla f_i(\bx)= 0 $.
Under such circumstances, directly incorporating stochastic gradient methods with gossip or global averaging may not converge to the global stationary solution \cite{nedic09dgd} without the strong assumption like bounded data dissimilarity.

To address this issue, gradient-tracking (GT) techniques \cite{lorenzo2016next,nedic2017achieving,qu2018harnessing} have been proposed, which utilizes gossip communication for global gradient estimation leveraging dynamic average consensus \cite{zhu2010discrete}. 
Recently, DSGT \cite{pu2021distributed} incorporates GT with stochastic gradient methods for computational efficiency. The updates at the $k$-th iteration are defined as: every agent $i\in[n]$ updates its optimization variable $\bx_i^k$ and gradient-tracking variable $\by_i^k$ by 
\begin{equation*}
    \begin{aligned}
    \bx^{k+1} &= \sum_{j=1}^n w_{ij}(\bx_j^{k}-\eta\by_j^k),\\
    \by_i^{k+1} &= \sum_{j=1}^n w_{ij}\by_j^{k} + \bg_i^{k+1}-\bg_i^{k},
    \end{aligned}
\end{equation*}
where $\eta>0$ is the step-size and the initialization $\by_i^0 = \bg_i^0$. In addition, \cite{xin2019variance} incorporates GT with variance-reduced techniques and \cite{li2020communication} develops an approximate Newton-type methods with variance-reduced GT to further accelerate the convergence. More recent works \cite{liu2023decentralized,ge2023gradient} prove that DSGT with multiple local updates is able to converge under high data heterogeneity in nonconvex setting.
\section{Proposed \alg~Algorithm}
In this section, we introduce \alg, which exploits communication heterogeneity from the probabilistic communication model and inherits the robustness to data heterogeneity from GT. Before the depiction of \alg, we first introduce some compact-form notations for convenience. Let the matrices $
\bX = \left[\bx_1,\bx_2,\ldots,\bx_n\right]\in\mathbb{R}^{d\times n}$ and $\bY = \left[\by_1,\by_2,\ldots,\by_n\right]\in\mathbb{R}^{d\times n}$ represent the collection of all the optimization variables and gradient-tracking variables, respectively.
We also denote the gradient of empirical loss given the sampled batch dataset $\mathcal{Z} = \{\mathcal{Z}_i\}_{i=1}^n$ as
\begin{align*}
    &\nabla \ell(\bX;\mathcal{Z}) = \left[\sum_{\bz_1\in\mathcal{Z}_1}\nabla \ell (\bx_1;\bz_1),\ldots,\sum_{\bz_n\in\mathcal{Z}_n}\nabla \ell (\bx_n;\bz_n)\right].
\end{align*}
 With the local stochastic gradient as \eqref{eq:stochastic_grad} in hand, the distributed stochastic gradient can be represented by 
$$\bG = \left[\bg_1,\bg_2,\ldots,\bg_n\right] = \frac{1}{b}\nabla \ell(\bX;\calZ).$$

\begin{algorithm}[!t] 
    \caption{\alg~for semi-decentralized nonconvex optimization}
    \begin{algorithmic}[1]
        \State \textbf{input:} $\bX^0=\bx^0\boldsymbol{1}_n^{\top}$, local-update and communication step sizes $\eta_l, \eta_c$, number of iterations $K$, number of local updates $T_o$, mini-batch size $b$.
       \State \textbf{initialize:} Draw the mini-batch $\mathcal{Z}^0 = \{\mathcal{Z}_i^0\}_{i=1}^{n}$ randomly  and  set $\boldsymbol{Y}^0 =  \bG^0 = \frac{1}{b} \nabla \ell(\bX^0;\calZ^0)$.
       \For{$k = 0,1,\cdots, K-1$}
       \State Set $\bX^{k+1,0} = \bX^k$, $\bY^{k+1,0} = \bY^k$ and $\bG^{k+1,0} = \bG^{k}$.
            \For{$t = 1,2\cdots,T_o$}%\Comment{$T_o$-local updates}
            \State Draw the mini-batch $\mathcal{Z}^{k+1,t}$ and compute
            \begin{subequations}\label{eq:local_update}
                \begin{align}
                        \bX^{k+1,t} &=  \bX^{k+1,t-1} -\eta_l \bY^{k+1,t-1}\label{eq:local_update_x}\\
                        \bG^{k+1,t} &= \frac{1}{b}\nabla \ell(\bX^{k+1,t};\calZ^{k+1,t})\\
                        \bY^{k+1,t} &= \bY^{k+1,t-1} + \bG^{k+1,t}-  \bG^{k+1,t-1}\label{eq:local_update_y}.
                \end{align}
                \end{subequations}
            \EndFor
            \State Define $\bW^k = \begin{cases}
                \bJ & \text{with probability}~p, \\
                \bW & \text{otherwise}.
            \end{cases}$%\Comment{probablistic heterogeneous communication}
            \State Draw the mini-batch $\mathcal{Z}^{k+1}$ and update
            \begin{subequations}\label{eq:communication_update}
                \begin{align}
                        \bX^{k+1} &= \left((1-\eta_c)\bX^{k} \!+\! \eta_c (\bX^{k+1,T_o}-\eta_l\bY^{k+1,T_o})\right)\bW^k\label{eq:communication_update_x}\\
                        \bG^{k+1} &= \frac{1}{b}\nabla \ell(\bX^{k+1};\calZ^{k+1})\\
                        \bY^{k+1} &= \left(\bY^{k+1,T_o} + \bG^{k+1}-  \bG^{k+1,T_o}\right)\bW^k\label{eq:communication_update_y}.
                \end{align}
            \end{subequations}
       \EndFor
    \end{algorithmic}
    \label{alg:GT-PCM}
\end{algorithm}

Then, we are ready to describe \alg~detailed in Algorithm \ref{alg:GT-PCM}, using the above compact notations. At the beginning of the $k$-th communication round, \alg~maintains the model estimate $\bX^k$, the global gradient estimate $\bY^k$ and the distributed stochastic gradient $\bG^k$. 
It then boils down to the following two stages for achieving both communication efficiency and exact convergence under data heterogeneity.
\begin{itemize}
    \item The first stage is to execute $T_o$ local steps without any communication (\cf~line 4-7). The key idea is to utilize the local computational resources to facilitate the convergence. 
    At the beginning of the local updates, initialize the local-update variables $\bX^{k+1,0} = \bX^k$, $\bY^{k+1,0} = \bY^k$ and $\bG^{k+1,0} = \bG^{k}$.
    At the $t$-th local update,  update $\{\bX^{k+1,t},\bY^{k+1,t},\bG^{k+1,t}\}$ via \eqref{eq:local_update}, maintaining the fashion of gradient-tracking techniques.
    \item The second stage is to perform the information exchange over the semi-decentralized network via a probabilistic communication model (\cf~line 8-10), \ie, there are two possible communication schemes --- agent-to-server communication with probability $p$ and agent-to-agent communication otherwise. Different schemes correspond to different mixing matrices (\cf~line 8), \ie, if agents implement the global communication, set $\bW^k = \bJ$; otherwise, set $\bW^k = \bW$. Then, agents update $\{\bX^{k+1},\bY^{k+1},\bG^{k+1}\}$ via \eqref{eq:communication_update}, using the output of local updates $\{\bX^{k+1,T_o},\bY^{k+1,T_o},\bG^{k+1,T_o}\}$ via the selected communication scheme.
\end{itemize}

\section{Theoretical Guarantees}
In this section, we provide the convergence results of our \alg~under different settings: \alg~converges at a rate of $O(1/\sqrt{nT_o K})$ using mini-batch gradients and $O(1/(nK))$ with full-batch gradients.

\subsection{Assumptions}
Before proceeding to the results, we first impose the following assumptions on the network model, objective functions and data sampling.
\begin{assump}[Semi-decentralized network model]\label{assump:graph}
    Given the undirected graph $\mathcal{G}$ and its mixing matrix $\bW$ following the Definition \ref{def:mixing_matrix}, then $\bW^k$ defined in Algorithm \ref{alg:GT-PCM} satisfies 
    \begin{equation*}
        \E[\|\bW^k\bx-\barx\|_2^2] \le  (1-\lambda_p)\|\bx-\barx\|_2^2, \quad\forall\bx\in\mathbb{R}^n,
    \end{equation*}
 where $\barx = \bJ\bx\in\mathbb{R}^n$ and the expected mixing rate $\lambda_p = \lambda_w + p(1-\lambda_w)\in(0,1]$.
\end{assump}
Note that Assumption \ref{assump:graph} is weaker than the connected assumption in prior semi-decentralized literatures \cite{chen2021accelerating,guo2022hybrid,wang2022accelerating}, \ie, $\lambda_w>0$. More specifically, Assumption \ref{assump:graph} implies that the underlying graph can be disconnected if and only if $p>0$.  Only in the case that the centralized server is unavailable (\ie, $p=0$), Assumption \ref{assump:graph} presumes the connectivity of $\mathcal{G}$.

Regarding the objective functions, we assume that the optimal value $f^{\star} \coloneqq \min_{\bx} f(\bx)$ exists and $f^\star>-\infty$. 
The local objective functions $\{f_i\}_{i=1}^n$ could be nonconvex but satisfy the standard smoothness assumption provided below.
\begin{assump}[$L$-smooth]\label{assump:smooth}
Each local function $f_i(\bx)$ is differentiable and there exists a constant $L$ such that 
\begin{equation*}
    \|\nabla f_i(\bx)-\nabla f_i(\by)\|_2 \le L\|\bx-\by\|_2, \quad \forall \bx,\by\in\mathbb{R}^d.
\end{equation*}
\end{assump}

In addition, we assume that the local stochastic gradient $\bg_i$ is an unbiased estimate with a bounded variance, which is widely used in the literature \cite{chen2021accelerating,wang2022accelerating,guo2022hybrid,ge2023gradient,ana21gt,zhao2022beer}.
\begin{assump}[Bounded variance]\label{assump:bounded_variance}
    For every agent $i\in[n]$, there exists a constant $\sigma\ge 0$ s.t. 
    $$\E_{\calZ_i\sim\calD_i}[\|\bg_i-\nabla f_i(\bx)\|_2^2]\le \sigma^2/b, \quad\forall \bx\in\mathbb{R}^d.$$
\end{assump}

Note that in the case of the full-batch gradients, \ie, the mini-batch size $b=m$, we can simply set $\sigma = 0$ and thus Assumption \ref{assump:bounded_variance} always holds.

%%%%%%%%%%%%%%%%%%%%%%%%%%%%%%%%%%%%%%%%%%%%%%%%%%%%
\subsection{Convergence analysis of \alg} 
Now, we are ready to present our main results.
First, the following theorem demonstrates that our proposed \alg~ is able to converge to the neighborhood of the stationary solution to the  problem \eqref{eq:prob} at the rate of $O(1/K)$ with constant step-sizes. The proof is postponed to the Appendix \ref{appendix:proof_thm1}.

    \begin{thm}[Convergence rate]\label{thm:stochastic}
        Suppose Assumption \ref{assump:graph}, \ref{assump:smooth} and \ref{assump:bounded_variance} hold. Let $\tilde{f} = f(\barx^0) - f^\star$  and $\bPhi_y^0 = \bY^0 - \bY^0\bJ$.
        For any $\alpha\ge 0.1$ s.t. $\eta_c = \alpha\sqrt{(1+p)}\lambda_p$ and $\eta_l \le \frac{\sqrt{(1+p)}\lambda_p}{360\alpha L(T_o+1)}$,
        it holds that  $\frac{1}{K}\sum_{k=0}^{K-1}\E[\|\nabla f(\barx^k)\|_2^2]$ converges at the rate of
        \begin{equation} \label{eq:thm1_rate}
        \begin{aligned}
            &\underset{\text{terms due to SGD and local updates}}{\underbrace{O\left(\frac{\tilde{f}}{\eta T_oK} +\left(L^2T_o^2\eta_l^2+\frac{L\eta}{n}\right)\frac{\sigma^2}{b}\right)}}+ \underset{\text{terms due to decentralized overhead}}{\underbrace{O\left(\frac{(1-p)L^2T_o^2\eta^2}{(1+p)^2\lambda_p^4}\frac{\sigma^2}{b}+ \frac{1}{nK}\E[\|\bPhi_y^0\|_{\normF}^2]\right)}},
        \end{aligned}
        \end{equation}
        where the average model estimate $\barx^k = \frac{1}{n}\sum_{i=1}^n \bx_i^k\in\mathbb{R}^d$ and $\eta = \eta_c\eta_l$.
    \end{thm}
Note that the above convergence result can hold even under significant data heterogeneity across agents, since we do not assume any bounded similarity between local objectives.

Due to the existence of the variance $\sigma^2$, we fine-tune the local-update step-size to obtain the exact convergence rate with the leading term $O(1/\sqrt{nT_o K})$, based on Theorem \ref{thm:stochastic}. Specifically, the following corollary considers the scenarios with mini-batch gradients (i.e., $b \le O(\sigma^2 K)$ and $\sigma>0$), while the case of large or full batch gradients (i.e., the batch size $b\ge \Theta(\sigma^2 K)$ or $\sigma = 0$) will be discussed later in Corollary \ref{corollary:large-batch}. The proof of Corollary \ref{corollary:communication_complexity_minibatch} is postponed to Appendix \ref{appendix:proof_corollary}.
    \begin{corollary}[Convergence rate with mini batch]\label{corollary:communication_complexity_minibatch} Suppose all the conditions in Theorem \ref{thm:stochastic} hold. Consider that the number of communication rounds $K$ is sufficiently large, i.e., $K\ge  \Theta(\frac{nbT_o L\tilde{f}}{\lambda_p^4\sigma^2})$,
        and the mini-batch size $b\le\!O(\sigma^2K)$, where $\sigma>\!0$. 
        If the step-sizes $\eta_c = \alpha\sqrt{1+p}\lambda_p$,  
        $\eta_l= \frac{1}{\alpha T_o}\min\big\{\sqrt{\frac{n\alpha^2bT_o\tilde{f}}{\eta_c^2L\sigma^2K}},\sqrt[3]{\frac{\eta_c b\tilde{f} }{\alpha L^2\sigma^2K}}\big\}$.
    Then, $\frac{1}{K}\sum_{k=0}^{K-1}\E[\|\nabla f(\barx^k)\|_2^2]$ can converge at the rate of $$O\left(\left(\frac{L\sigma^2}{nT_obK}\right)^{\frac{1}{2}}\!\!+ \left(\frac{L\sigma}{\lambda_p^2\sqrt{b}K}\right)^{\frac{2}{3}} \!\!+ \!\frac{1}{nK}\right).$$
    \end{corollary}
    From Corollary \ref{corollary:communication_complexity_minibatch}, \alg~can achieve the \textit{$\epsilon$-accuracy}, \ie, $\frac{1}{K}\sum_{k=0}^{K-1}\E[\|\nabla f(\barx^k)\|_2^2]\le\epsilon^2$ after
            \begin{align*}\label{eq:communication_complexity}
                O\left(\frac{L\sigma^2}{nT_ob\epsilon^4} + \frac{L\sigma}{\lambda_p^2\sqrt{b}\epsilon^3}+ \frac{1}{n\epsilon^2} \right)
            \end{align*}
    communication rounds.
    Notice that if $K$ is sufficiently large and the first term $\left(\frac{L\sigma^2}{nT_obK}\right)^{\frac{1}{2}}$ correspondingly becomes dominant, increasing the number of agents $n$ or the number of local updates $T_o$ can accelerate the convergence. Such a linear speedup matches the findings in the special cases of semi-decentralized ML, \ie, decentralized setting \cite{liu2023decentralized,ge2023gradient} when $p=0$ and federated setting \cite{karimireddy2020scaffold} when $p=1$. 
    
    In fact, \alg~can be generalized to the decentralized case and federated case by setting $p=0$ and $p=1$ respectively, while maintaining comparable convergence guarantees.
    \begin{table*}[!t]
        \centering
        \renewcommand{\arraystretch}{1.5}
         \begin{tabular}{|c | c |c |} 
         \hline
         Algorithm  & \# Agent-to-server communication & \# Agent-to-agent communication \\
         \hline
         \begin{tabular}{@{}c@{}}SCAFFOLD\\\cite{karimireddy2020scaffold} \end{tabular}  & $O\left(\dfrac{\sigma^2}{nT_o\epsilon^4}+\dfrac{1}{\epsilon^2}\right)$ & 0\\
         \hline
         \begin{tabular}{@{}c@{}}LSGT\\\cite{ge2023gradient}\end{tabular}   &0 & $\begin{aligned}
            O\left(\dfrac{\sigma^4}{nT_o\lambda_w^8\epsilon^4} + \dfrac{1}{nT_o^{1/3}\lambda_w^{8/3} \epsilon^{4/3}}+ \dfrac{1}{nT_o\epsilon^2}\right)\end{aligned}$\\
         \hline
         \begin{tabular}{@{}c@{}}Periodical-GT\\\cite{liu2023decentralized}\end{tabular}  &0  & $O\left(\dfrac{\sigma^2}{nT_o\epsilon^4} +\dfrac{\sigma}{\lambda_w^2\epsilon^3} + \dfrac{1}{\lambda_w^2\epsilon^2}\right)$ \\
         \hline
         \begin{tabular}{@{}c@{}}$K$-GT \\\cite{liu2023decentralized}\end{tabular} &0  & $O\left(\dfrac{\sigma^2}{nT_o\epsilon^4} +\dfrac{\sigma}{\lambda_w^2\sqrt{T_o}\epsilon^3} + \dfrac{1}{\lambda_w^2\epsilon^2}\right)$ \\
         \hline
         \rowcolor{Gray}
         \textbf{This paper} & $\begin{aligned}
            &O\left(\dfrac{p\sigma^2}{nT_o\epsilon^4} + \dfrac{p\sigma}{(\lambda_w+p)^2\epsilon^3}+ \dfrac{p}{n\epsilon^2}\right)\\
        \end{aligned}$
          &$\begin{aligned}
            &O\left(\dfrac{(1-p)\sigma^2}{nT_o\epsilon^4} + \dfrac{(1-p)\sigma}{(\lambda_w+p)^2\epsilon^3}+  \dfrac{1-p}{n\epsilon^2}\right)
        \end{aligned}$
        \\
         \hline
         \end{tabular}
         \caption{The number of the expected agent-to-server/agent communication rounds of ours and existing decentralized and federated ML algorithms with stochastic gradients and local updates, to achieve the $\epsilon$-accuracy,  where 
         $\epsilon$ is sufficiently small. Here, $n$ is the number of agents, $T_o$ is the number of local updates, $\lambda_w$ is the mixing rate of the underlying graph.}
         \label{tab:communication_rounds}
    \end{table*}
    \begin{figure}[!t]
        \centering
        \includegraphics[width = 0.6\linewidth,trim={5cm 14cm 10cm 10cm},clip]{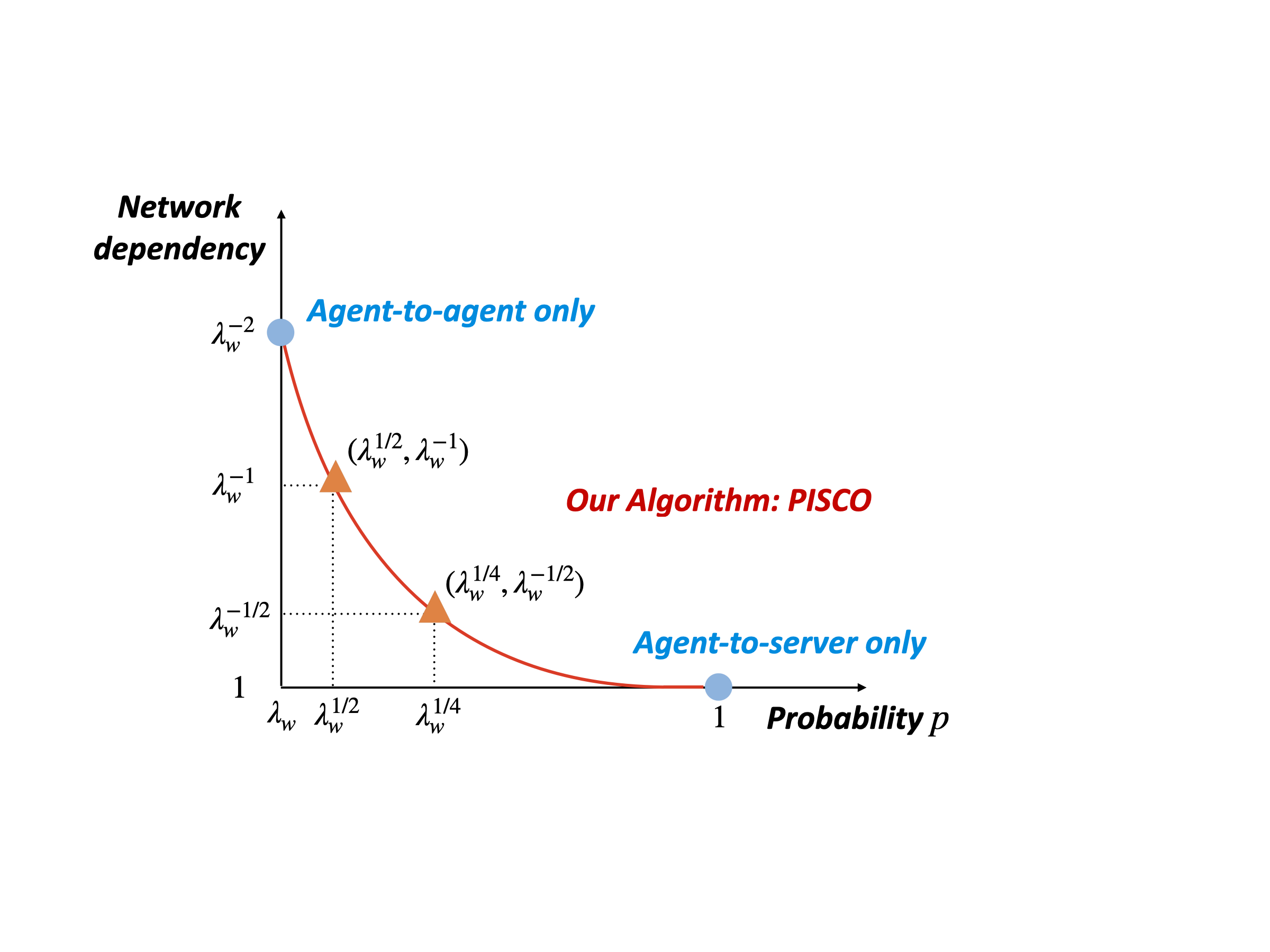}
        \caption{The network dependency of  \alg~regarding agent-to-server communication probability $p$.}
        \label{fig:network_dependency}
    \end{figure}
     
\begin{remark}[Decentralized case]
    When $p=0$,  Algorithm \ref{alg:GT-PCM} becomes fully decentralized, \ie, agents only perform local communication. Then, the communication complexity to achieve $\epsilon$-accuracy becomes
    \begin{equation*}
        O\left(\frac{L\sigma^2}{nT_ob\epsilon^4} + \frac{L\sigma}{\lambda_w^2\sqrt{b}\epsilon^3}+ \frac{1}{n\epsilon^2} \right),
    \end{equation*}
    which is better than the rate of Periodical-GT in \cite{liu2023decentralized} and LSGT \cite{ge2023gradient}, since the network dependency is $O(1/\lambda_w^2)$ and only appears in the second term (see Table \ref{tab:communication_rounds}).
    The second term is slightly worse than $K$-GT, the variance-reduced Periodical-GT \cite{liu2023decentralized},  since it corrects the descent direction with the average of $T_o$ local updates instead of the last local update at communication. However, they require that the initial local correction variables are settled in a centralized way. Combining our analysis with such a variance-reduction method while avoiding the centralized initialization would be a promising future direction of this paper. 
\end{remark}

\begin{remark}[Federated case]
When $p=1$, every agent can communicate with the server directly and thus \alg~performs in the federated fashion at any iteration $k\ge 1$. Then, the communication complexity becomes
\begin{equation*}
    O\left(\frac{L\sigma^2}{nT_ob\epsilon^4} + \frac{L\sigma}{\sqrt{b}\epsilon^3}+ \frac{1}{n\epsilon^2} \right),
\end{equation*}
where the leading term is the same as that of SCAFFOLD in \cite{karimireddy2020scaffold} with the linear speedup in terms of the network size $n$ and the number of local updates $T_o$.
\end{remark}

Moreover, the highlight of our work is to fill the void of semi-decentralized ML with the probabilistic communication model and gain the best communication efficiency from both agent-to-agent communication and agent-to-server communication, as shown in Table \ref{tab:communication_rounds}.
In addition, \alg~is able to improve the network dependency of the communication overheads from $O(\lambda_w^{-2})$ \cite{liu2023decentralized} to $O(\lambda_p^{-2})$, where the trade-off between the communication probability and the network dependency is illustrated in Figure \ref{fig:network_dependency}. The flexible heterogeneous communication brings the superior communication efficiency of \alg~in both well-connected and poorly-connected networks.

\begin{remark}[For well-connected networks]
As gossip communication is efficient to mix information for well-connected networks, \alg~is able to achieve a comparable convergence rate with much fewer agent-to-server communication rounds compared with using only agent-to-server communication. Therefore, our \alg~can significantly reduce the communication costs for well-connected networks whenever local agent-to-agent communications are inexpensive.
\end{remark}

\begin{remark}[For poorly-connected  networks] When $\lambda_w\to 0$, performing agent-to-agent only communication often results in a large number of communication rounds and prohibitive communication costs. As shown in Figure \ref{fig:network_dependency}, with any probability $p \ge \lambda_w$, the network dependency can be reduced to $O(p^{-2})$. More specifically, even a small agent-to-server probability $p = \Theta(\sqrt{\lambda_w})$ can significantly improve the network dependency from $O(\lambda_w^{-2})$ to $O(\lambda_w^{-1})$. Take the large-scale path graph as an example, where the mixing rate $\lambda_w$ scales on the order of $O(1/n^2)$~\cite{nedich2018networktopology}. Our \alg~with $p= \Theta(1/n)$ can improve the network dependency from $O(n^4)$ to $O(n^2)$. Moreover, if $p=\Theta(1)$, the communication complexity can be network-independent like Gossip-PGA \cite{chen2021accelerating}, but our theoretical analysis does not require the additional assumption of the bounded dissimilarity between local objectives.
\end{remark}

In many real-world scenarios, it is also popular to choose large mini-batch size $b$ to guarantee the exact convergence to the stationary point. As the terms related to the variance $\sigma^2$ on the right hand side of \eqref{eq:thm1_rate} scale on the order of $O(\sigma^2/b)$, if we choose a large enough mini-batch size $b \ge \Theta(\sigma^2/\epsilon^2)$, the following desirable result holds.
    \begin{corollary}[Communication complexity with large batch]\label{corollary:large-batch}
            Suppose all the conditions in Theorem \ref{thm:stochastic} holds.
            If the batch size $b$ is sufficiently large, \ie, $b\ge \Theta(\frac{\sigma^2}{\epsilon^2})$, it holds  $\frac{1}{K}\sum_{k=0}^{K-1}\E[\|\nabla f(\bar{\bx}^k)\|_2^2]\le\epsilon^2$ after
            \begin{equation*}
                O\left(\frac{L}{(1+p)\lambda_p^2\epsilon^2}+\frac{1}{\epsilon^2}\right)
            \end{equation*}
            communication rounds. In addition, if the mini-batch size $b = m$, \ie, we take the full-batch gradient, the comunication complexity will be improved to
            \begin{equation*}
                O\left(\frac{L}{(1+p)\lambda_p^2\epsilon^2}+\frac{1}{n\epsilon^2}\right).
            \end{equation*}
    \end{corollary}

Note that Corollary \ref{corollary:large-batch} also matches the result in decentralized setting \cite{nguyen2022performance} and federated setting \cite{karimireddy2020scaffold}, by setting $p=0$ and $p=1$ respectively.

\section{Numerical Experiments}

In this section, we present the numerical performance of  \alg~on real-world datasets, to substantiate its superior performance in terms of communication efficiency and robustness to various topologies and data heterogeneity.

\subsection{Logistic regression with nonconvex regularization}\label{subsection:logistic}
To investigate communication efficiency of \alg, we conduct experiments on logistic regression with a nonconvex regularization term \cite{wang2019spiderboost} using the \texttt{a9a} dataset \cite{chang2011libsvm}. Given the model parameter $\bx$ and data sample $\bz = (\boldsymbol{a},y)$, the empirical loss $\ell(\bx;\bz)$ is defined as:
\begin{equation*}
    \ell(\bx;\bz) = \log\left(1+\exp(-y\boldsymbol{a}^{\top}\bx)\right) + \rho\sum_{l=1}^d \frac{\bx(l)^2}{1+\bx(l)^2},
\end{equation*} 
where $\boldsymbol{a}\in\mathbb{R}^{d}$ is the feature vector, $y\in\{-1,1\}$ is the corresponding label, the regularizer coefficient $\rho$ is set as $0.01$, and $\bx(l)$ denotes the $l$-th coordinate of $\bx$.

In this subsection, we consider a ring topology with $n=10$ agents and evenly partition the sorted \texttt{a9a} dataset to $10$ agents to augment the data heterogeneity. Roughly speaking, every agent will receive  $m=3256$ training samples of dimension $d=124$, where $5$ agents will receive data with label $1$ and the others will receive data with label $0$. Regarding the mixing matrix, we follow the symmetric FDLA matrix \cite{xiao2004fast} to aggregate information among neighbors. In addition, we set the batch size $b=256$ for the following experiments. To reduce the impact of randomness, we run every experiment with 5 different seeds and show the average results.

\paragraph{The impact of different agent-to-server probabilities.}
 First, we study the influence of the agent-to-server communication probability $p$ on the training and test performance. 
 To this end, we vary the probability $p$ from $\{1,1/10^{0.5},1/10^{0.75},1/10,1/10^{1.25},1/10^{1.5},1/10^{1.75},1/10^{2}, 0\}$ and present the number of communication rounds of \alg~with different $p$ to achieve $0.05$ training accuracy (\ie, $\frac{1}{K}\sum_{k=0}^{K-1}\|\nabla f(\barx^k)\|_2^2\le 0.05$) and $80\%$ test accuracy ($\ge 95\%$ of the peak accuracy within 1000 communication rounds), in Figure \ref{fig:iter_p}.

 \begin{figure}[!t]
    \centering 
    \begin{subfigure}[b]{0.475\textwidth}
    \centering
    \includegraphics[width = \textwidth]{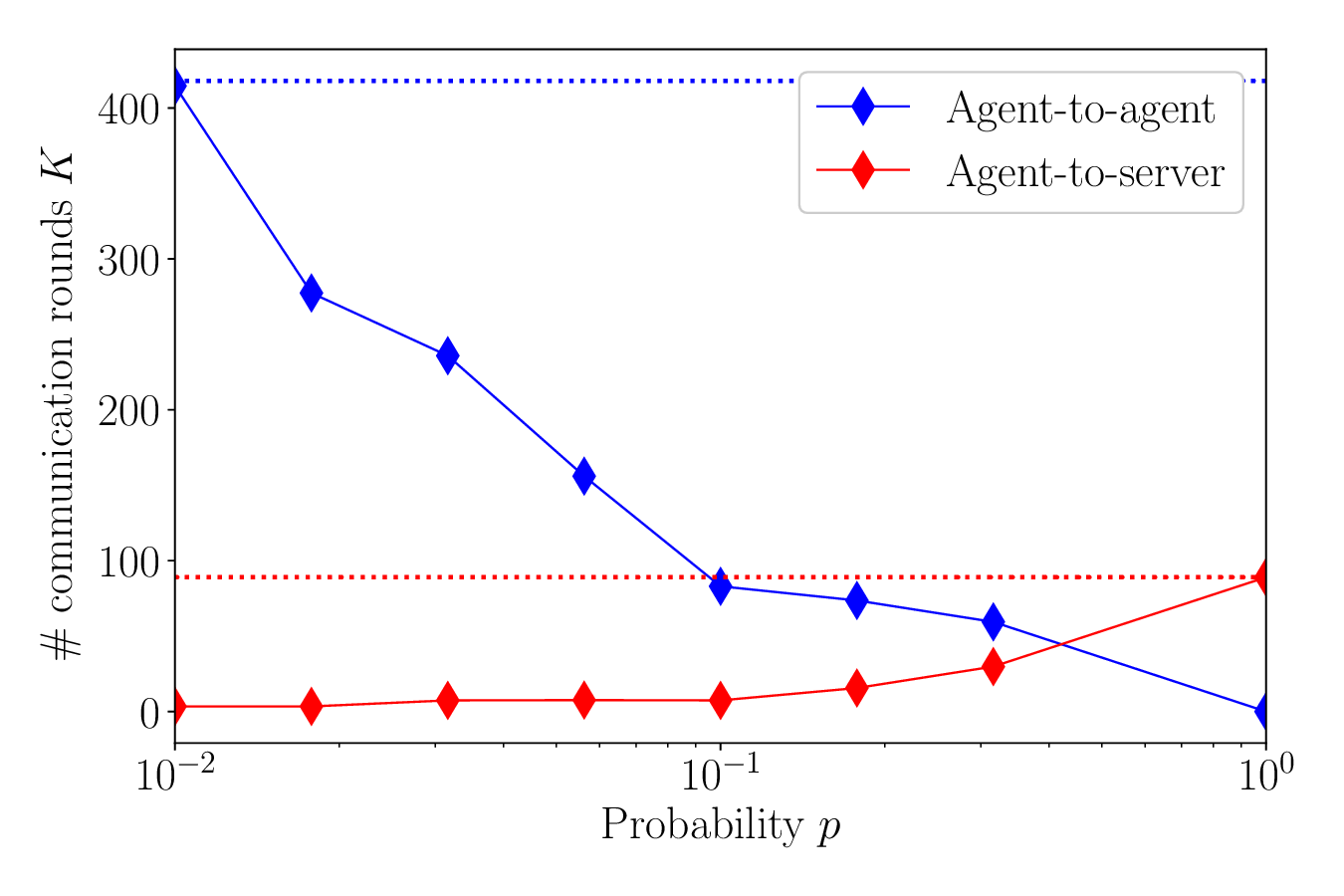}
    \caption{Training accuracy $\frac{1}{K+1}\sum_{k=0}^K\|\nabla f(\barx^k)\|_2^2\le 0.05$.}
    \label{fig:iter_p_trian}
    \end{subfigure}
    \qquad
    \begin{subfigure}[b]{0.475\textwidth}
    \centering
    \includegraphics[width = \textwidth]{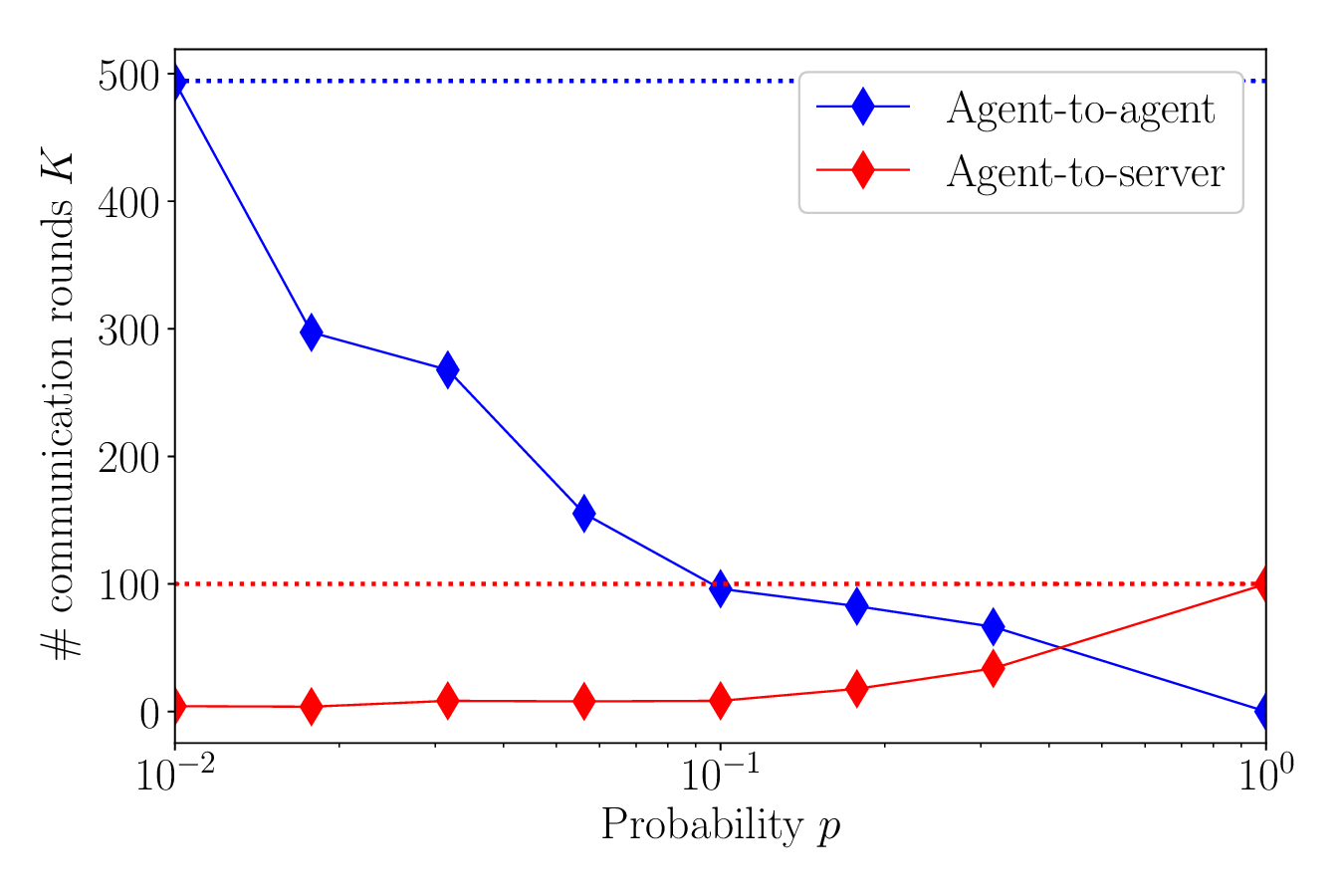}
    \caption{Test accuracy $\le 80\%$.}
    \label{fig:iter_p_test}
    \end{subfigure}
    \caption{The number of agent-to-agent and agent-to-server communication rounds required to achieve $0.05$ training accuracy (the left panel) and  $80\%$ test accuracy (the right panel) for \alg~with $T_o=1$ and different $p \in\{0,10^{-2},10^{-1.75},10^{-1.5}, 10^{-1.25},10^{-1},10^{-0.75},10^{-0.5},1\}$. Here, the blue (red) dotted line represents the number of agent-to-agent (agent-to-server) communication rounds that \alg~with $p=0$ (with $p=1$) requires.} 
    \label{fig:iter_p}
\end{figure}

From Figure \ref{fig:iter_p}, we observe that just a small agent-to-server probability (\eg, $p\le 0.1$) can considerably reduce the number of communication rounds required to attain a specific accuracy during both training and testing phases. For instance, \alg~with $p=10^{-1.25}\approx 0.06$ can reduce agent-to-agent communication rounds by $60\%$, with several agent-to-server communication rounds. Moreover, even if the server is more accessible (\eg, $p\ge 0.1$), increasing the agent-to-server communication probability $p$ might not further save the total communication rounds. This indicates that not all costly communications between agents and the server are crucial for accelerating the convergence compared with decentralized methods. Therefore, by leveraging heterogeneous communication, we can reduce the average per-round communication expense while preserving a comparable rate of convergence.

\paragraph{The speedup of multiple local updates.}
 To verify the speedup of multiple local updates, we plot the training accuracy and test accuracy of \alg~with different numbers of local updates $T_o=1$ (\cf~Figure \ref{fig:K_1}) and $T_o=10$ (\cf~Figure \ref{fig:K_10}). In both cases, we vary the probability $p\in\{1,10^{-0.5},10^{-1},0\}$. It is worth noting that with only $p = 0.1$ or $p=10^{-0.5}$, \alg~already achieves almost the same performance as \alg~with $p=1$. Comparing Figure \ref{fig:K_10} with Figure \ref{fig:K_1}, we can clearly observe the speedup brought by multiple local updates for different probabilities. For example, for \alg~with $p=0.1$, the number of communication rounds required to attain $0.05$ training accuracy or $80\%$ testing accuracy decreases roughly by $50\%$ if we increase $T_o$ from $1$ to $10$.  

\begin{figure}[!t]
    \centering 
    \begin{subfigure}[b]{\textwidth}
    \centering
    \includegraphics[width = 0.9\textwidth]{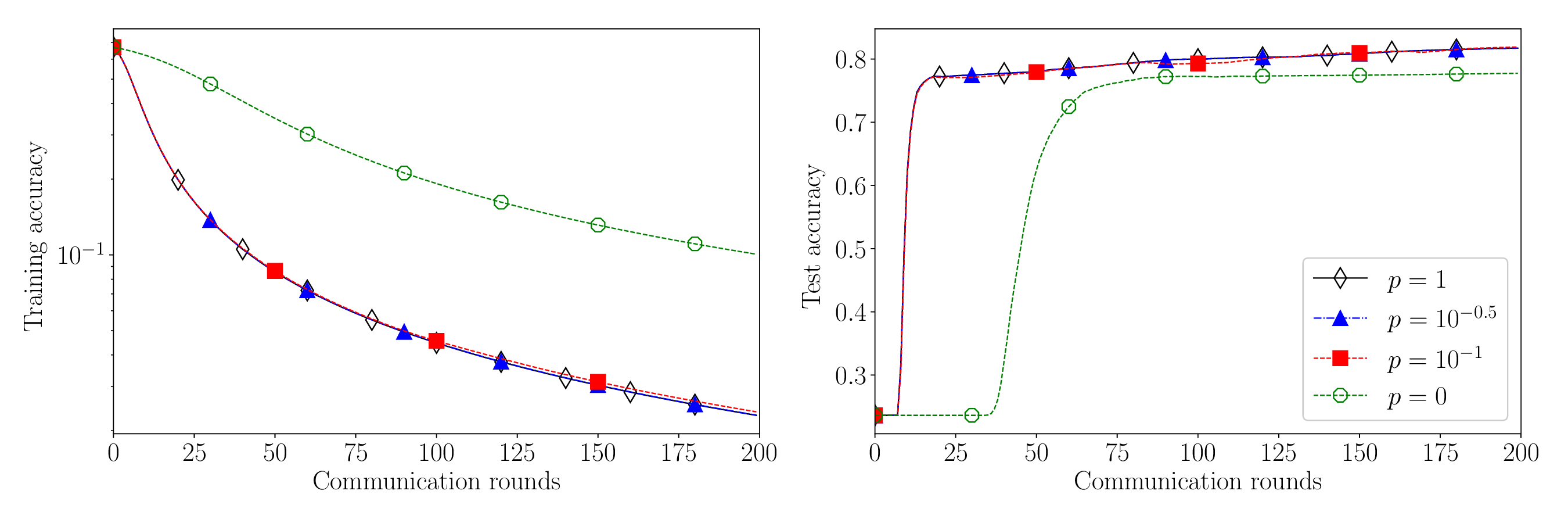}
    \caption{The number of local updates $T_o=1$.}
    \label{fig:K_1}
    \end{subfigure}
    \\
    \begin{subfigure}[b]{\textwidth}
    \centering
    \includegraphics[width = 0.9\textwidth]{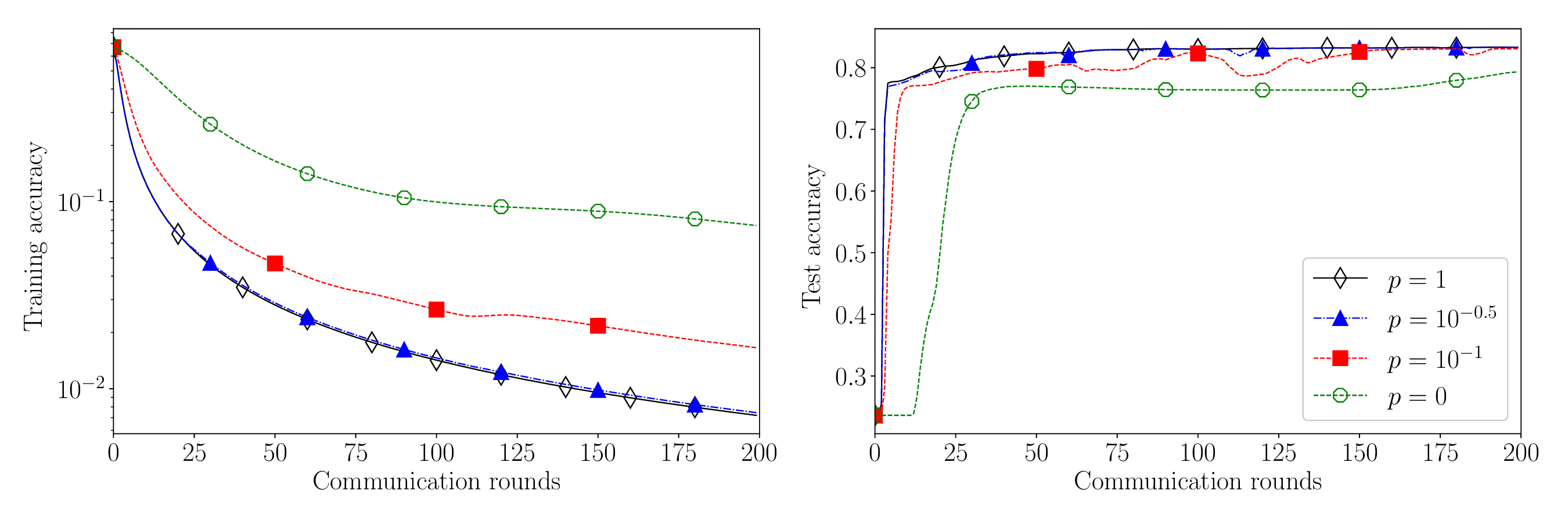}
    \caption{The number of local updates $T_o=10$.}
    \label{fig:K_10}
    \end{subfigure}
    \caption{The training accuracy (left two panels) and testing accuracy (right two panels) against communication rounds with different probabilities $p=1,10^{-0.5},10^{-1},0$ and different number of local updates $T_o=1,10$, over a ring topology for logistic regression with a nonconvex regularizer on the sorted \texttt{a9a} dataset.} 
    \label{fig:different_p}
\end{figure}

\subsection{Neural network training}
\begin{figure}[h]
    \centering 
    \begin{subfigure}[b]{\textwidth}
        \centering
        \includegraphics[width = 0.9\textwidth]{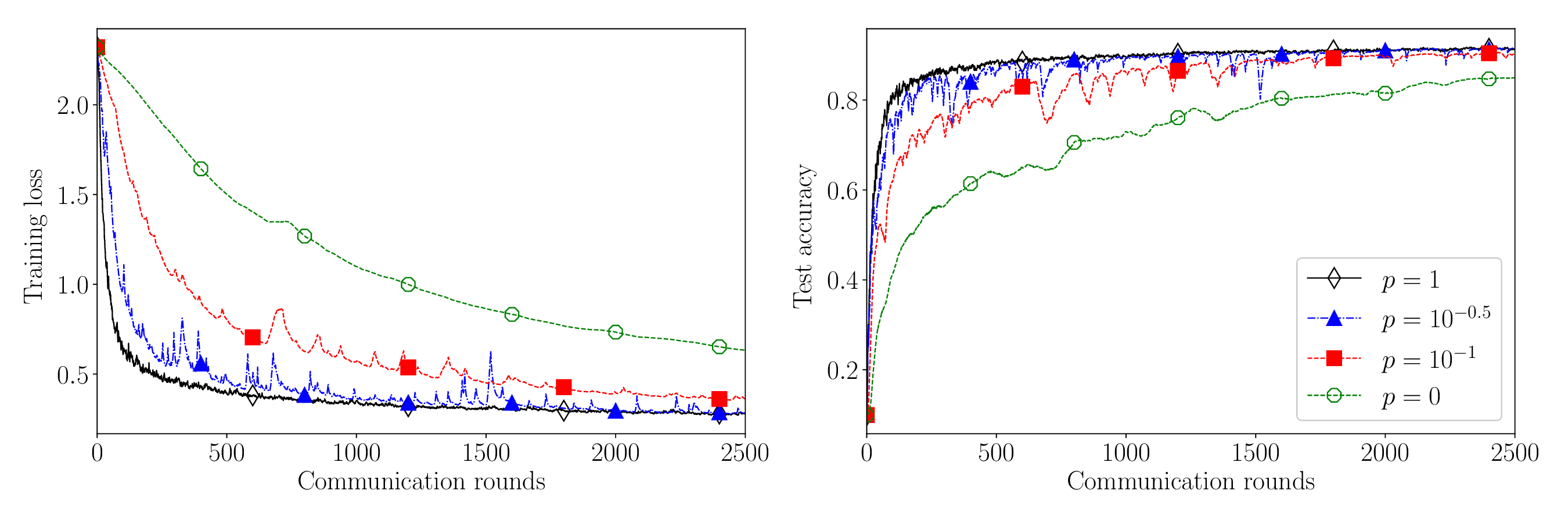}
        \caption{A connected Erd\H{o}s-R\'enyi graph with $0.3$ connectivity probability.}
        \label{fig:er3}
    \end{subfigure}
    \\
    \begin{subfigure}[b]{\textwidth}
    \centering
    \includegraphics[width = 0.9\textwidth]{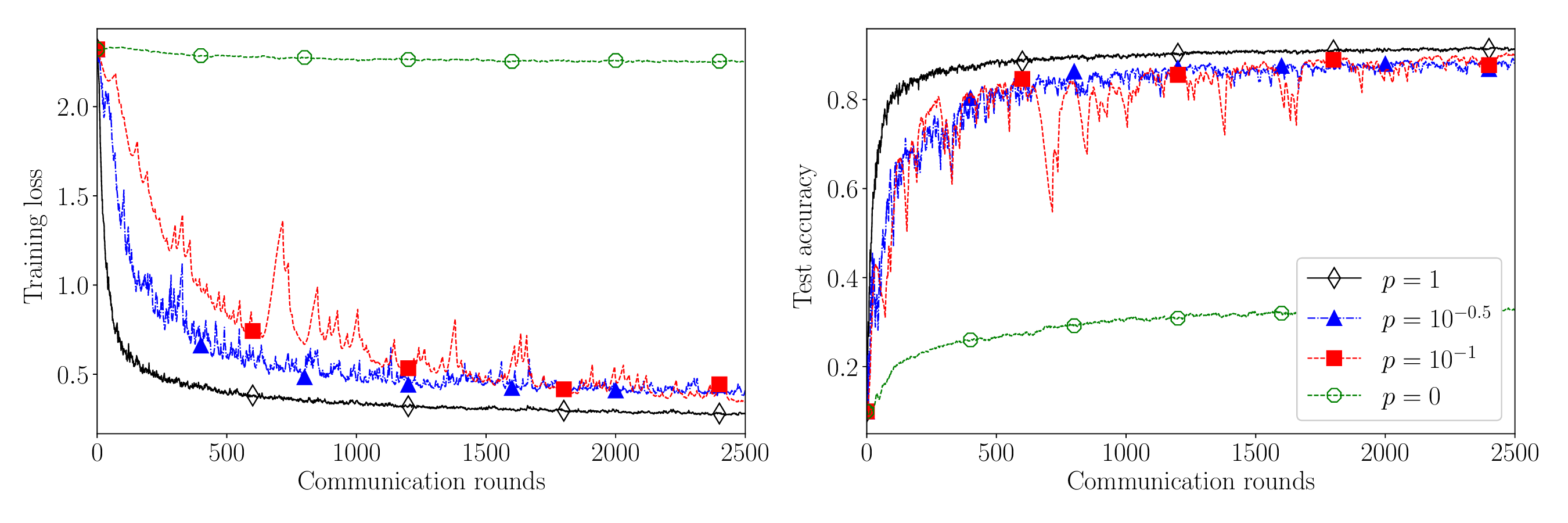}
    \caption{A disconnected Erd\H{o}s-R\'enyi graph with $0.1$ connectivity probability.}
    \label{fig:dis}
    \end{subfigure}
    \caption{The training loss (the left two panels) and testing accuracy (the right two panels) against communication rounds with different probabilities $p=1,10^{-0.5},10^{-1},0$ and the number of local updates $T_o=10$ over both well-connected and  disconnected Erd\H{o}s-R\'enyi graphs for 1-hidden-layer network training on the sorted \texttt{MNIST} dataset.} 
    \label{fig:different_topology}
\end{figure}

Further, we run the single hidden-layer neural network training with 32 hidden neurons on the  \texttt{MNIST} dataset \cite{li2012minist}.
More specifically, we use the sigmoid and softmax function as the activation function, where the empirical loss w.r.t. the training parameter $\bx = \text{vec}\left(\bW_1,\boldsymbol{c}_1,\bW_2,\boldsymbol{c}_2\right)$ and the sample $\bz = (\boldsymbol{a},y)$ is defined using the cross entropy loss as:
\begin{equation*}
\mathtt{CrossEntropy}\!\left(\mathtt{softmax}\!\left(\bW_2~\mathtt{sigmoid}(\bW_1\boldsymbol{a}\!+\!\boldsymbol{c}_1)\!+\!\boldsymbol{c}_2\right)\!,y\right)\!,
\end{equation*}
where the training weights $\bW_1\in\mathbb{R}^{32\times 784}$, $\bW_2\in\mathbb{R}^{10\times 32}$, $\boldsymbol{c}_1\in\mathbb{R}^{32}$, and $\boldsymbol{c}_2\in\mathbb{R}^{10}$.

To verify the robustness of \alg~to diverse topologies, we consider a well-connected Erd\H{o}s-R\'enyi topology with a connectivity probability of  $0.3$ (corresponding to $\lambda_w=0.38$) and a disconnected Erd\H{o}s-R\'enyi topology with a connectivity probability of $0.1$  (corresponding to $\lambda_w=0$). To simulate the highly data-heterogeneous scenario, we evenly split the sorted \texttt{MNIST} dataset to $n=10$ agents, where agent $i\in[n]$ will receive the training data associated with label $i$. Moreover, we set the batch size $b=100$, the number of local updates $T_o=10$ and the agent-to-server communication probability $p \in\{1,1/\sqrt{n},1/n,0\} = \{1,10^{-0.5},10^{-1},0\}$. To reduce the impact of randomness, we run every experiment with 3 different seeds and show the average results in Figure \ref{fig:different_topology}.

In Figure \ref{fig:different_topology}, our \alg~shows impressive robustness to high data heterogeneity and different topologies, including the well-connected network (\cf~Fig \ref{fig:er3}) and the disconnected network (\cf~Figure \ref{fig:dis}). By comparing Figure \ref{fig:er3} with Figure \ref{fig:dis}, we observe that better connectivity makes gossip communication sufficiently efficient to mix information. As a result, heterogeneous communication with a smaller $p$ can attain a comparable performance to that of \alg~with $p=1$ in Figure \ref{fig:er3}. Notice that the performance of \alg~with no agent-to-server communication degenerates remarkably when the network is disconnected. In contrast, semi-decentralized \alg~(\ie, $0<p<1$) maintains performance levels similar to \alg~with $p=1$. It illustrates that a few number of agent-to-server communication rounds can largely mitigate the impact of the network connectivity, even for disconnected graphs.

We also evaluate the performance of \alg~by training a convolutional neural network (CNN) on the unshuffled \texttt{CIFAR10} dataset \cite{krizhevsky2009learning}. The network architecture includes three sequential CNN modules, each containing two 2D convolutional layers with ReLU activation, followed by max pooling (kernel size 2, stride 2) and dropout (rate 0.2) for regularization. Specifically, in the first module, the initial convolutional layer transforms the input from 3 to 32 channels, and the second convolutional layer maintains 32 channels; the second module follows this pattern, mapping 32 to 64 channels; and the third similarly increases from 64 to 128 channels. After feature extraction, the fully connected layers process the 2048 flattened output, first mapping it to 128 features with ReLU activation and dropout, and then to 10 outputs for classification. We use a ring topology with $n=5$ agent and set the batch size $b=20$ and the number of local updates $T_o=4$. To introduce data heterogeneity,  we split the sorted \texttt{CIFAR10} dataset across the 5 agents, so that each agent $i\in[n]$ obtains training data with label $i$ and $i+5$.

In Figure \ref{fig:cnn}, we illustrate the effectiveness of \alg, in terms of training loss and test accuracy across epochs. We observe that, due to sparse agent-to-agent communication in the ring topology and extremely high data heterogeneity, \alg~with $p=0$ converges more slowly than \alg~with $p>0$. Notably, \alg~with $p=1/\sqrt{5}$ achieves performance comparable to \alg~with $p=1$, demonstrating the efficiency of the heterogeneous communication protocol in reducing costly agent-to-server communications.

\begin{figure*}[!t]
    \centering 
    {\includegraphics[width = 0.9\linewidth]{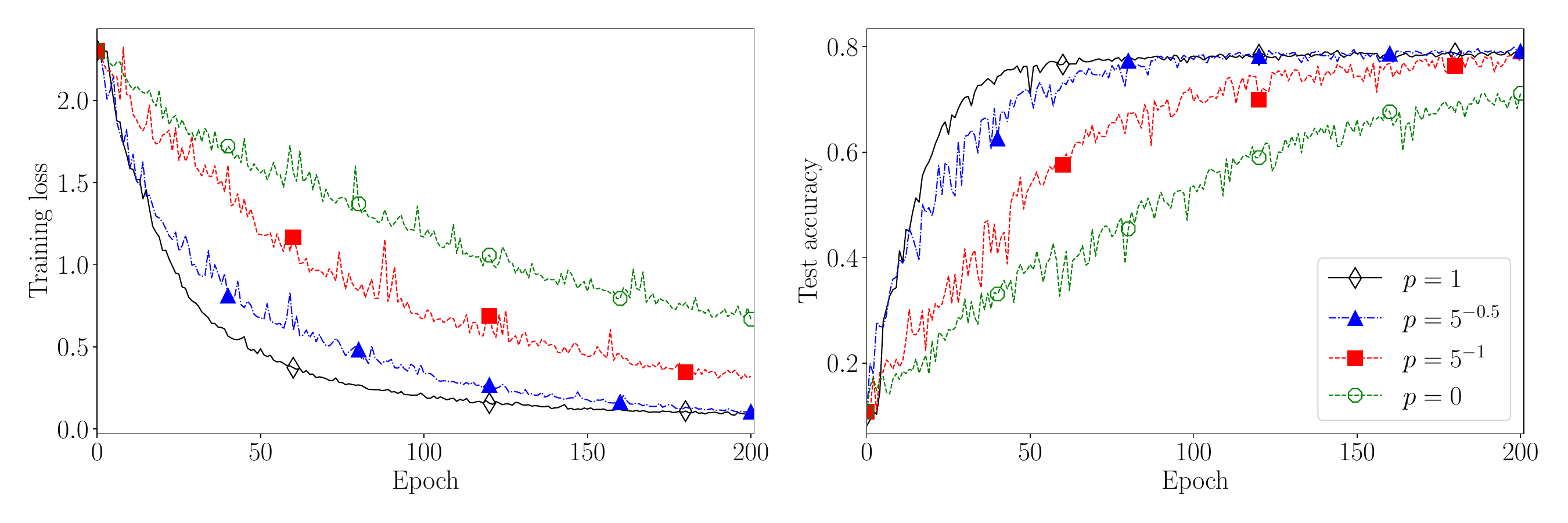}
    }
    \caption{The training loss and testing accuracy across epochs  with different probabilities $p=1,1/\sqrt{5},0.2,0$ and the number of local updates $T_o=4$, over a ring topology for CNN training on the sorted \texttt{CIFAR10} dataset.} 
    \label{fig:cnn}
\end{figure*}

\section{Conclusions}
In this paper, we develop a communication-efficient algorithm called \alg~for solving federated nonconvex optimization over semi-decentralized networks, which enjoys the linear speedup of local updates and addresses data dissimilarity without any additional assumptions. By leveraging the heterogeneous communication model, \alg~largely reduces communication overheads in terms of the network dependency with a few agent-to-server communication rounds, particularly evident in poorly-connected networks. Both theoretical guarantees and empirical experiments underscore \alg's outstanding communication efficiency and robustness to data heterogeneity and various network topologies. 

In the future, it will be of interest to incorporate variance reduction techniques \cite{li2020communication,bollapragada2018adaptive} into the algorithm design, apply communication compression \cite{zhao2022beer}  to further reduce the per-round communication costs, and enable varying agent-to-server communication probabilities \cite{saha2024privacy}, allowing for personalized and heterogeneous communication strategies for each agent.

 \section*{Acknowledgement}
 This work is supported in part by the grants ONR N00014-19-1-2404, NSF CIF-2007911, CNS-2148212, ECCS-2318441, AFRL FA8750-20-2-0504, and is supported in part by funds from federal agency and industry partners as specified in the Resilient \& Intelligent NextG Systems (RINGS) program. He Wang is also gratefully supported by the Bob Lee Gregory Fellowship at Carnegie Mellon University. 
%%%%%%%%%%%%%%%%%%%%%%%%%%%%%%%%%%%%%%%%%%%%%%%%%%%%%%%%%%%%
\bibliographystyle{alphaabbr} 
\bibliography{references,distributed}

\appendix
\section{Technical Lemmas}
This section establishes several critical lemmas which will
be used in the proof of Theorem 1.
 For the ease of the analysis, we begin with rewriting the updates of \alg~ in a more compact form. 
First, by unrolling the local updates \eqref{eq:local_update} in \alg, we have
\begin{align}
    \bX^{k+1,T_o} &= \bX^{k}-\eta_l\bY^k-\eta_l\sum_{t=1}^{T_o-1} \bY^{k+1,t},\label{eq:unroll_local_x}\\
    \bY^{k+1,T_o} &= \bY^k +  \sum_{t=1}^{T_o}\left( \bG^{k+1,t}-\bG^{k+1,t-1}\right)\nonumber.
\end{align}
Then, summing \eqref{eq:local_update_y} from $t=1$ to $T_o-1$ and telescoping results in
\begin{align}
    \sum_{t=1}^{T_o-1} \bY^{k+1,t} &= \sum_{t=1}^{T_o-1}\bY^{k+1,t-1} + \sum_{t=1}^{T_o-1}\left( \bG^{k+1,t} - \bG^{k+1,t-1}\right)\nonumber\\
    & = \bY^{k} + \sum_{t=1}^{T_o-2} \bY^{k+1,t-1} + \sum_{t=1}^{T_o-2}\left( \bG^{k+1,t}- \bG^{k+1,t-1}\right) + \sum_{t=1}^{T_o-1}\left( \bG^{k+1,t}- \bG^{k+1,t-1}\right)\nonumber\\
    &\shortvdotswithin{=}
    & =(T_o-1)\bY^k + \sum_{l=1}^{T_o-1}\sum_{t=1}^l \left(\bG^{k+1,t}- \bG^{k+1,t-1}\right).\label{eq:sum_local_y}
\end{align}
Substituting the above equation \eqref{eq:sum_local_y} into \eqref{eq:unroll_local_x} and combining with \eqref{eq:communication_update}, the updates of \alg~can be rewritten as the following compact form:
\begin{subequations}\label{eq:compact-update}
\begin{align}
    \bX^{k+1} &= \left(\bX^{k}-\eta (T_o+1)\bY^k-\eta \sum_{l=1}^{T_o}\sum_{t=1}^l \left(\bG^{k+1,t}- \bG^{k+1,t-1}\right)\right)\bW^k\nonumber\\
    & = \left(\bX^k-\eta(T_o+1)\bY^k - \eta\sum_{t=1}^{T_o} \bG^{k+1,t} + \eta T_o\bG^{k}\right)\bW^k,\label{eq:compact-update-x}\\
    \bY^{k+1} & =  \left(\bY^k+ \bG^{k+1}-\bG^{k+1,T_o} +  \sum_{t=1}^{T_o}\left( \bG^{k+1,t}-\bG^{k+1,t-1}\right) \right)\bW^k\nonumber\\
    & = \left(\bY^k+ \bG^{k+1} - \bG^k\right)\bW^k,\label{eq:compact-update-y}
\end{align}
\end{subequations}
where $\eta = \eta_c\eta_l$. 

\subsection{Useful inequalities}\label{section:useful_inequalities}
Before proceeding, we introduce the following propositions which will be useful in the analysis.
\begin{proposition}\label{propo:sum_of_varaibles_decompose}
Let $\{\bv_i\}_{i=1}^M$ be a set of $M$ vectors in $\mathbb{R}^d$. Then, for any $0<\delta \le 1$,
\begin{align}
    \|\sum_{i=1}^M \bv_i\|_2^2 &\le M\sum_{i=1}^M \|\bv_i\|_2^2,\label{eq:proposition_sum}\\
    \|\bv_i+\bv_j\|_2^2 &\le (1+\delta)\|\bv_i\|_2^2 +  (1+\frac{1}{\delta})\|\bv_j\|_2^2\label{eq:proposition_jessen}.
\end{align}
\end{proposition}

\begin{proposition}\label{propo:matrix_multiply_decompose} For $\boldsymbol{A},\boldsymbol{B}\in\mathbb{R}^{d\times n}$, $\|\boldsymbol{A}\boldsymbol{B}\|_{\F}\le \|\boldsymbol{A}\|_{\F}\|\boldsymbol{B}\|_2$.
\end{proposition}
\begin{proposition}[\cite{karimireddy2020scaffold}]\label{propo:sum_of_noise}
    Let $\{\boldsymbol{\Xi}_i\}_{i=1}^n$ be $n$ random variables in $\mathbb{R}^d$, which may not be independent of each other. Suppose that $\E[\boldsymbol{\Xi}_i|\boldsymbol{\Xi}_{i-1},\ldots,\boldsymbol{\Xi}_1]=\bxi_i$ and  $\E\left[\|\boldsymbol{\Xi}_i-\bxi_i\|^2_2\right] \le \sigma^2$. Then, we have
    $$\E\left[\|\sum_{i=1}^n \left(\boldsymbol{\Xi}_i-\bxi_i\right)\|_2^2\right] \le n\sigma^2.$$
\end{proposition}

\subsection{Property of gradient tracking}
In this section, the following lemma shows that $\barY^k = \bY^k\bJ$ is able to track the average of local stochastic gradients, \ie, $\barG^k =\bG^k\bJ =  \left(\frac{1}{n}\sum_{i=1}^n \bg_i^k\right)\boldsymbol{1}_n^{\top}$. We define the average model estimate as $\barx^k = \frac{1}{n}\sum_{i=1}^n\bx_i^k\in\mathbb{R}^d$ and $\barX^k = \bX^k\bJ$.

\begin{lemma}[Gradient tracking property]\label{lemma:y_tracking}
    Suppose Assumption \ref{assump:graph} holds. Then for any $k\in\mathbb{N}$, 
    \begin{equation*}\label{eq:tracking_property}
        \barY^k  =  \barG^k.
    \end{equation*}
    In addition, if Assumption \ref{assump:smooth} and \ref{assump:bounded_variance} hold, we have
    \begin{equation}\label{eq:barY}
        \E[\|\barY^k\|_{\normF}^2] \le \frac{3\sigma^2}{b} + 3L^2\E[\|\bPhi_x^k\|_{\normF}^2] + 3n\E[\|\nabla f(\barx^k)\|_2^2],
    \end{equation}
    where $\bPhi_x^k = \bX^k-\barX^k$ is the consensus error at iteration $k$.
\end{lemma}

\begin{proof}
    From the equation \eqref{eq:compact-update-y},
\begin{equation*}%\label{eq:proof_y_tracking1}
    \bY^{k+1}\bJ= \left(\bY^k+ \bG^{k+1} - \bG^{k}\right) \bW^k\bJ.                  
\end{equation*}
Because of Assumption \ref{assump:graph}, \ie, $\bW^k\mathbf{1}_n = \mathbf{1}_n$, and the fact that $\barY^{k+1} = \bY^{k+1}\bJ$,
\begin{equation*}
\barY^{k+1} =\barY^{k}  + \barG^{k+1} - \barG^k.                      
\end{equation*}
Summing from $k=0$ to $t-1$ gives
\begin{align*}\barY^{t} &= \barY^0 + \barG^{t} - \barG^0  =  \barG^{t},
\end{align*}
due to the initialization of Algorithm \ref{alg:GT-PCM} s.t. $\bY^0 =  \bG^{0}$. %Additionally, we define the gradient noise at the $t$-th iteration as $\bS^t = \bG^t - \nabla F(\bX^t)$ and the average noise as $\barS^t = \bS^t\bJ = \barG^t - \overline{\nabla F}(\bX^t)\mathbf{1}_n^{\top}$, which leads to \eqref{eq:tracking_property}.
Moreover, we define the distributed objective as $F(\boldsymbol{X}) = \sum_{i=1}^n f_i(\bx_i)$, where the distributed gradient is denoted by $\nabla F(\boldsymbol{X}) = \left[\nabla f_1(\bx_1), \nabla f_2(\bx_2), \ldots, \nabla f_n(\bx_n)\right]\in\mathbb{R}^{d\times n}.$
Then, we have
\begin{align}
        \E\left[\big\|\barY^{k}\big\|_{\normF}^2\right] & = \E\left[\big\|\barG^k\big\|_{\normF}^2\right]= \E\left[\big\|\bG^k\bJ\big\|_{\normF}^2\right]\nonumber\\
        &=\E\left[\big\|\left(\barG^k - \nabla F(\bX^k) + \nabla F(\bX^k) - \nabla F(\barX^k) + \nabla F(\barX^k)\right)\bJ\big\|_{\normF}^2\right] \nonumber\\
        &\overset{(a)}{=}n\E\Bigl[\big\|\frac{1}{n}\sum_{i=1}^n \left(\bg_i^k - \nabla f_i(\bx_i^k)+\nabla f_i(\bx_i^k) - \nabla f_i(\barx^k)+\nabla f_i(\barx^k)\right)\big\|_2^2\Bigr]\nonumber\\
        &\overset{(b)}{=}\frac{3}{n}\E\Bigl[\big\|\sum_{i=1}^n \left(\bg_i^k - \nabla f_i(\bx_i^k)\right)\big\|_2^2\Bigr] + 3\sum_{i=1}^n\E\left[\big\| \nabla f_i(\bx_i^k) - \nabla f_i(\barx^k)\big\|_2^2\right] +  3n\E\left[\big\|\nabla f(\barx^k)\big\|_2^2\right]\nonumber\\
        &\overset{(c)}{\le} 3\sigma^2/b + 3L^2\sum_{i=1}^n\E[\|\bx_i^k-\barx^k\|_2^2] + 3n\E[\|\nabla f(\barx^k)\|_2^2]\nonumber\\
        &= 3 \sigma^2/b + 3L^2\E\left[\|\bPhi_x^k\|_{\normF}^2\right] + 3n\E\left[\|\nabla f(\barx^k)\|_2^2\right],\label{eq:bar_Y}
\end{align}
where $(a)$ is based on $\|\bx\mathbf{1}_n^{\top}\|_{\normF}^2 = n\|\bx\|_2^2$, $(b)$ uses \eqref{eq:proposition_sum} in Proposition \ref{propo:sum_of_varaibles_decompose}, $(c)$ follows Assumption \ref{assump:smooth} and \ref{assump:bounded_variance} and applies Proposition \ref{propo:sum_of_noise}.

\end{proof}

\subsection{Progress improvement on successive iterates and averages}
The following two auxiliary lemmas bound the progress improvement between the successive iterates and their averages. Similar to $\bPhi_x^k = \bX^k - \barX^k$, we also use $\bPhi_y^k = \bY^k -\barY^k$ to represent the tracking error at the $k$-th iteration. As for the local updates, we define $t$-th local-update consensus error as $\bPhi_x^{k,t} = \bX^{k+1,t} - \barX^k$ and tracking error as $\bPhi_y^{k,t} = \bY^{k+1,t} - \barY^k$.

\begin{lemma}[Progress improvement between successive iterates]\label{lemma:progress_improvement_iterates} Suppose Assumption \ref{assump:graph}, \ref{assump:smooth} and \ref{assump:bounded_variance} hold. Then, we have
    \begin{align*}
        \E[\|\bX^{k}-\bX^{k-1}\|_{\normF}^2]
        &\le 12\left(1+2T_o^2L^2\eta^2\right)\E[\|\bPhi_x^{k-1}\|_{\normF}^2]+6(1-p)\lambda^2(T_o+1)^2\eta^2\E[\|\bPhi_y^{k-1}\|_{\normF}^2]+\frac{48nT_o^2\eta^2\sigma^2}{b}\\
        &+24T_o L^2\eta^2\sum_{t=1}^{T_o}\E[\|\bPhi_x^{k-1,t}\|_{\normF}^2]+3(T_o+1)^2\eta^2\E[\|\barY^{k-1}\|_{\normF}^2].
    \end{align*}
\end{lemma}
\begin{proof}
    Define $\bDelta_x^{k-1} = \bX^k-\bX^{k-1}\bW^{k-1}$. Then,
    \begin{align}
        \E[\|\bX^{k}-\bX^{k-1}\|_{\normF}^2]\nonumber
        &= \E\left[\big\|\bX^{k-1}\left(\bW^{k-1}-\bI\right)+\bDelta_x^{k-1}\big\|_{\normF}^2\right]\nonumber\\
        &= \E\left[\big\|\left(\bX^{k-1}-\barX^{k-1}\right)\left(\bW^{k-1}-\bI\right)+\bDelta_x^{k-1} + \eta(T_o+1)\barG^{k-1}-\eta(T_o+1)\barG^{k-1}\big\|_{\normF}^2\right]\nonumber\\
        &\overset{(a)}{\le} 3 \E\left[\big\|\bPhi_x^{k-1}\big\|_{\normF}^2\big\|\bW^{k-1}-\boldsymbol{I}\big\|_2^2 + \big\|\bDelta_x^{k-1}+\eta(T_o+1)\barG^{k-1}\big\|_{\normF}^2 +\eta^2(T_o+1)^2\big\|\barG^{k-1}\big\|_{\normF}^2\right]\nonumber\\
        &\overset{(b)}{\le} 12\E[\|\bPhi_x^{k-1}\|_{\normF}^2] + 3\E[\|\bDelta_x^{k-1}+\eta(T_o+1)\barG^{k-1}\|_{\normF}^2] +3\eta^2(T_o+1)^2\E[\|\barG^{k-1}\|_{\normF}^2],\label{eq:progress_successive_origin}
    \end{align}
    where $(a)$ is from Proposition \ref{propo:sum_of_varaibles_decompose} and $(b)$ is due to $\|\bW^{k-1}-\bI\|_2\le 2$, for any $k\ge 1$.
    
    Next, we are going to bound the second term $ \E[\|\bDelta_x^{k-1}+\eta(T_o+1)\barG^{k-1}\|_{\normF}^2]$. Because of \eqref{eq:compact-update-x}, we have 
    $$\bDelta_x^{k-1} = -\eta\left((T_o+1)\bY^{k-1} + \sum_{t=1}^{T_o}\bG^{k,t} - T_o\bG^{k-1}\right)\bW^{k-1}.$$
    Then,
    \begin{align}
        &\E[\|\bDelta_x^{k-1}+\eta(T_o+1)\barG^{k-1}\|_{\normF}^2]\nonumber\\
        &= \eta^2\E\left[\big\|(T_o+1)\left(\bY^{k-1}-\barY^{k-1}\right)\left(\bW^{k-1}-\bJ\right) + \sum_{t=1}^{T_o} \left(\bG^{k,t}-\bG^{k-1}\right)\bW^{k-1}\big\|_{\normF}^2\right]\nonumber\\
        &\le 2\eta^2\left((T_o+1)^2\E[\|\bPhi_{y}^{k-1}\|_{\normF}^2\|\bW^{k-1}-\bJ\|_2^2]+\E\left[\|\sum_{t=1}^{T_o} \left(\bG^{k,t}-\bG^{k-1}\right)\|_{\normF}^2\|\bW^{k-1}\|_2^2\right]\right)\nonumber\\
        &\le 2\eta^2\left((1-p)\lambda^2(T_o+1)^2\E[\|\bPhi_{y}^{k-1}\|^2_{\normF}]+T_o\sum_{t=1}^{T_o}\E\left[ \|\bG^{k,t}-\bG^{k-1}\|_{\normF}^2\right]\right)\label{eq:process_term2_origin},
    \end{align}
    where the first equality is based on Lemma \ref{lemma:y_tracking}, the penultimate inequality follows the Proposition \ref{propo:sum_of_varaibles_decompose} and \ref{propo:matrix_multiply_decompose} and the last one is from Proposition \ref{propo:sum_of_varaibles_decompose} and Assumption \ref{assump:graph}, \ie, $\E[\|\bW^{k-1}-\bJ\|_2^2] = (1-p)\lambda^2$.
    Here,
    \begin{align}
        &\E\left[\big\|\bG^{k,t}-\bG^{k-1}\big\|_{\normF}^2\right]\nonumber\\
        &= \E\left[\bG^{k,t} \!-\! \nabla F(\bX^{k,t}) \!+\! \nabla F(\bX^{k,t}) \!-\! \nabla F(\barX^{k-1})\!-\! \bG^{k-1} +\nabla F(\bX^{k-1}) \!-\! \nabla F(\bX^{k-1}) \!+\! \nabla F(\barX^{k-1})\big\|_{\normF}^2\right]\nonumber\\
        &\overset{(a)}{\le} 4\left(\E[\|\bG^{k,t} - \nabla F(\bX^{k,t})\|_\normF^2] + \E[\|\bG^{k-1} - \nabla F(\bX^{k-1})\|_{\normF}^2] + L^2\E[\|\bPhi_x^{k-1,t}\|_{\normF}^2] + L^2\E[\|\bPhi_x^{k-1}\|_{\normF}^2]\right) \nonumber\\
        &\overset{(b)}{\le} \frac{8n\sigma^2}{b} + 4L^2 \E[\|\bPhi_x^{k-1,t}\|_{\normF}^2] + 4L^2\E[\|\bPhi_x^{k-1}\|_{\normF}^2],\label{eq:gradient_improvement_local_update}
    \end{align}
    where $(a)$ is based on Proposition \ref{propo:sum_of_varaibles_decompose}  and Assumption \ref{assump:smooth} and $(b)$ is from Assumption \ref{assump:bounded_variance}. 
    Then, 
    \begin{equation}\label{eq:progress_term2}
    \begin{aligned}
        \E[\|\bDelta_x^{k-1}+\eta(T_o+1)\barG^{k-1}\|_{\normF}^2] &\le 2(1-p)\lambda^2(T_o+1)^2\eta^2\E[\|\bPhi_{y}^{k-1}\|^2_{\normF}]+ \frac{16nT_o^2 \eta^2\sigma^2}{b}\\
        &\quad+ 8T_o L^2\eta^2\sum_{t=1}^{T_o} \E[\|\bPhi_x^{k-1,t}\|_{\normF}^2] + 8T_o^2L^2\eta^2\E[\|\bPhi_x^{k-1}\|_{\normF}^2].
    \end{aligned}
    \end{equation}
    Finally, substituting \eqref{eq:progress_term2} into \eqref{eq:progress_successive_origin} and applying Lemma \ref{lemma:y_tracking} result in
    \begin{align*}
        \E[\|\bX^{k}-\bX^{k-1}\|_{\normF}^2] &\le 12\left(1+2T_o^2L^2\eta^2\right)\E[\|\bPhi_x^{k-1}\|_{\normF}^2]+6(1-p)\lambda^2(T_o+1)^2\eta^2\E[\|\bPhi_y^{k-1}\|_{\normF}^2]\\
        &\quad+\frac{48nT_o^2\eta^2\sigma^2}{b}+24T_o L^2\eta^2\sum_{t=1}^{T_o}\E[\|\bPhi_x^{k-1,t}\|_{\normF}^2]+3(T_o+1)^2\eta^2\E[\|\barY^{k-1}\|_{\normF}^2].
    \end{align*}
\end{proof}

%%%%%%%%%%%%%%%%%%%%%%%%%%%%%%%%%%%%%%
\begin{lemma}[Progress improvement between the averages]\label{lemma:progress_improvement_average}
    Suppose Assumption \ref{assump:graph}, \ref{assump:smooth} and \ref{assump:bounded_variance} hold. Then, we have
    \begin{equation*}\label{eq:diff_barx_error}
    \begin{aligned}
         \E\!\left[\big\|\barX^{k+1}-\barX^k\big\|_{\normF}^2\right] &\le \frac{3T_o\eta^2\sigma^2}{b}\!+\!3T_oL^2\eta^2\!\sum_{t=0}^{T_o}\E\left[\big\|\bPhi_x^{k,t}\big\|_{\normF}^2\right] +3nT_o^2\eta^2\E\left[\big\|\nabla f(\barx^k)\big\|_2^2\right].
    \end{aligned}
    \end{equation*}
    \end{lemma}

\begin{proof}
    From Lemma \ref{lemma:y_tracking} and the \eqref{eq:compact-update-x}, we have
    \begin{equation}\label{eq:diff_barx_compact}
        \barX^{k+1} = \barX^k -\eta\sum_{t=0}^{T_o} \barG^{k+1,t}.
    \end{equation}
    Similar to \eqref{eq:bar_Y}, we have
    \begin{align*}
        &\frac{1}{\eta^2}\E\left[\big\|\barX^{k+1}-\barX^k\big\|_{\normF}^2\right]\\
        &= \E\Bigl[\big\|\sum_{t=0}^{T_o}\barG^{k+1,t}\big\|_{\normF}^2\Bigr]\\
        &= n \E\Bigl[\big\|\frac{1}{n}\sum_{i=1}^n\sum_{t=0}^{T_o} \left(\bg_i^{k+1,t}- \nabla f_i(\bx_i^{k+1,t}) + \nabla f_i(\bx_i^{k+1,t}) - \nabla f_i(\barx^k) + \nabla f_i(\barx^k)\right)\big\|_2^2\Bigr]\\
        &\le \frac{3}{n} \E\Bigl[\big\|\sum_{i=1}^n\sum_{t=0}^{T_o}\left(\bg_i^{k+1,t} -\nabla f_i(\bx_i^{k+1,t})\right)\big\|_2^2\Bigr] + 3T_o \sum_{i=1}^n\sum_{t=0}^{T_o}\E\left[\big\|\nabla f_i (\bx_i^{k+1,t}) \!-\! \nabla f_i(\barx^k)\big\|_2^2\right] + 3nT_o^2 \E\left[\big\|\nabla f(\barx^k)\big\|_2^2\right]\\
        &\le \frac{3T_o\sigma^2}{b} + 3T_oL^2 \sum_{i=1}^n\sum_{t=0}^{T_o} \E\left[\big\|\bx_i^{k+1,t}-\barx^k\big\|_2^2\right] + 3nT_o^2 \E\left[\big\|\nabla f(\barx^k)\big\|_2^2\right] \\
        &\le \frac{3T_o\sigma^2}{b}+3T_oL^2\sum_{t=0}^{T_o}\E\left[\big\|\bPhi_x^{k,t}\big\|_{\normF}^2\right]+3nT_o^2\E\left[\big\|\nabla f(\barx^k)\big\|_2^2\right].
    \end{align*}
\end{proof}
\subsection{Bounding consensus errors and tracking errors}
Next, we present the following lemma for bounding the accumulated consensus errors for local updates, in order to control the consensus and tracking errors at every iteration.

\begin{lemma}[Accumulated consensus drift for local updates]\label{lemma:consensus_local_drift}
Suppose Assumption \ref{assump:smooth} and \ref{assump:bounded_variance} hold.
If $\eta_l \le \frac{1}{8L(T_o+1)}$, we have
\begin{align*}
    \sum_{t=1}^{T_o}\E[\|\bPhi_x^{k,t}\|_{\normF}^2]&\le 9T_o\E[\|\bPhi_x^k\|_{\normF}^2]+ 8\eta_l^2(T_o+1)^3\E[\|\bPhi_y^k\|_{\normF}^2]+\!\frac{64n\eta_l^2 T_o^3 \sigma^2}{b} \!+\! 3\eta_l^2 T_o (T_o+1)^2\E[\|\barY^k\|_{\normF}^2].
\end{align*}
\end{lemma}
\begin{proof}
    For the $t$-th local update at the $k$-th iteration, it follows \eqref{eq:local_update_x} that 
    \begin{align}
        \bPhi_x^{k,t} &= \bX^{k+1,t} - \barX^k = \bX^k - \eta_l \bY^k - \eta_l\sum_{l=1}^{t-1} \bY^{k+1,l} -\barX^k+t\eta_l\barY^k-t\eta_l\barY^k\nonumber\\
        &= \bX^k- \barX^k - \eta_l(\bY^k - \barY^k) -\eta_l\sum_{l=1}^{t-1}(\bY^{k+1,l}-\barY^k)-t\eta_l\barY^k \nonumber\\
        &= \bPhi_x^k-\eta_l \bPhi_y^k -\eta_l \sum_{l=1}^{t-1}\bPhi_y^{k,l} -t\eta_l\barY^k. \label{eq:consensus_drift_localupdates}
    \end{align}
    
    To bound the accumulated consensus errors for $T_o$ local updates, we sum the the expectation of the squared norm of \eqref{eq:consensus_drift_localupdates} over $t$ from $t=1$ to $T_o$, leading to
    \begin{align}\label{eq:total_consensus_drift_1}
         \sum_{t=1}^{T_o} \E[\|\bPhi_x^{k,t}\|_{\normF}^2] &\overset{\eqref{eq:proposition_sum}}{\le} 4T_o\E[\|\bPhi_x^k\|_{\normF}^2] + 4\eta_l^2T_o\E[\|\bPhi_y^k\|_{\normF}^2] + 4\eta_l^2\sum_{t=2}^{T_o}(t-1)\sum_{l=1}^{t-1}\E[\|\bPhi_y^{k,l}\|_{\normF}^2] + \sum_{t=1}^{T_o}4\eta_l^2t^2 \E[\|\barY^k\|_{\normF}^2]\nonumber\\
         &~\le 4T_o\E[\|\bPhi_x^k\|_{\normF}^2] + 4\eta_l^2T_o\E[\|\bPhi_y^k\|_{\normF}^2] + 4\eta_l^2 \sum_{t=1}^{T_o-1} \frac{(T_o+t-1)(T_o-t)}{2}\E[\|\bPhi_y^{k,t}\|_{\normF}^2]\nonumber\\ 
         &~\quad+ \frac{2\eta_l^2 T_o(T_o+1)(2T_o+1)}{3}\E[\|\barY^k\|_{\normF}^2],
    \end{align}
    where the last inequality uses $\sum_{t=1}^{T_o} t^2 = \frac{T_o(T_o+1)(2T_o+1)}{6}$.
    
    From \eqref{eq:local_update_y}, we have
    $$\bY^{k+1,t} = \bY^{k}+\bG^{k+1,t}-\bG^{k},$$ such that
    \begin{align}
        \E[\|\bPhi_y^{k,t}\|_{\normF}^2] &~= \E[\|\bY^{k}+ \bG^{k+1,t} - \bG^{k}-\barY^k\|_{\normF}^2]\le 2\E[\|\bPhi_y^k\|_{\normF}^2] + 2\E[\| \bG^{k+1,t} - \bG^{k}\|_{\normF}^2]\nonumber\\
        &\overset{\eqref{eq:gradient_improvement_local_update}}{\le} 2\E[\|\bPhi_y^k\|_{\normF}^2] + 16n\sigma^2/b+8L^2\E[\|\bPhi_x^{k,t}\|_{\normF}^2]  + 8L^2\E[\|\bPhi_x^{k}\|_{\normF}^2].\label{eq:delta_y_local}
    \end{align}
    Due to the fact $\frac{(T_o+t-1)(T_o-t)}{2}\le \frac{T_o^2}{2}$ for $1\le t\le T_o$,
    \begin{equation*}
    \sum_{t=1}^{T_o-1} \frac{(T_o+t-1)(T_o-t)}{2}\E[\|\bPhi_y^{k,t}\|_{\normF}^2]\le{T_o^3} \left(\E[\|\bPhi_y^k\|_{\normF}^2] + 8n\sigma^2/b+4L^2\E[\|\bPhi_x^k\|_{\normF}^2]\right)
    +4L^2\sum_{t=1}^{T_o}T_o^2\E[\|\bPhi_x^{k,t}\|_{\normF}^2].
    \end{equation*}
    Combining with \eqref{eq:total_consensus_drift_1}, we have
    \begin{align*}
        \sum_{t=1}^{T_o}\E[\|\bPhi_x^{k,t}\|_{\normF}^2] &\le (4T_o+16\eta_l^2T_o^3 L^2)\E[\|\bPhi_x^k\|_{\normF}^2] + (4\eta_l^2T_o+4\eta_l^2T_o^3) \E[\|\bPhi_y^k\|_{\normF}^2] + 32n\eta_l^2 T_o^3 \sigma^2/b\\ 
        &\quad+ 16\eta_l^2 L^2 \sum_{t=1}^{T_o}T_o^2\E[\|\bPhi_x^{k,t}\|_{\normF}^2] +\frac{2\eta_l^2 T_o(T_o+1)(2T_o+1)}{3}\E[\|\barY^k\|_{\normF}^2].
    \end{align*}
    Therefore, if $\eta_l \le \frac{1}{8L(T_o+1)}$, we have
    \begin{align*}
        &\sum_{t=1}^{T_o}\E[\|\bPhi_x^{k,t}\|_{\normF}^2]\\
        &\le 8T_o(1+4\eta_l^2T_o^2 L^2)\E[\|\bPhi_x^k\|_{\normF}^2]+ 8\eta_l^2T_o(T_o^2+1)\E[\|\bPhi_y^k\|_{\normF}^2]+64n\eta_l^2 T_o^3 \sigma^2/b + 3\eta_l^2 T_o(T_o+1)^2\E[\|\barY^k\|_{\normF}^2].
    \end{align*}
\end{proof}

%%%%%%%%%%%%%%%%%%%%%%%%%%%%%%%%%%%%%%%%%%

With Lemma \ref{lemma:consensus_local_drift} in hand, we are ready to bound the consensus error $\E[\|\bPhi_x\|_{\normF}^2]$ and the tracking error $\E[\|\bPhi_y\|_{\normF}^2]$, respectively.
\begin{lemma}[Consensus error for communication updates]\label{lemma:consensus_error_communication_updates}
Suppose Assumption \ref{assump:graph},\ref{assump:smooth} and \ref{assump:bounded_variance} hold. If $\eta_l \le  \frac{1}{8L(T_o+1)}$ and $\eta\le\frac{\lambda_p}{80L(T_o+1)}$, we have
\begin{align*}
    \E[\|\bPhi_x^k\|_\normF^2] &< (1-p)\left[\frac{1+(1+p)\lambda^2}{2}\E[\|\bPhi_x^{k-1}\|_\normF^2]+\frac{40\lambda^2}{\lambda_p}(T_o+1)^2\eta^2\E[\|\bPhi_y^{k-1}\|_\normF^2] \right.\\
    &\quad+\frac{240\lambda^2L^2(T_o+1)^4\eta^2\eta_l^2}{\lambda_p}\E[\|\barY^{k-1}\|_\normF^2] \left.+ \frac{320\lambda^2n(T_o+1)^2\eta^2}{\lambda_p}\frac{\sigma^2}{b}\right].
\end{align*}
\end{lemma}
\begin{proof}
    Let $\bDelta_x^{k-1} = \bX^{k}-\bX^{k-1}\bW^{k-1}$ and $\overline{\bDelta}_x^{k-1} = \barX^{k}-\barX^{k-1}$. Then, 
    \begin{align}
       \E[\|\bPhi_x^k\|_\normF^2]  &\le \E[\|(\bX^{k-1} - \barX^{k-1})(\bW^{k-1}-\bJ)+\bDelta_x^{k-1} - \overline{\bDelta}_x^{k-1}\|_\normF^2]\nonumber\\
       &\overset{(a)}{\le} (1+\delta_1)\E[\|\bPhi_x^{k-1}\|_\normF^2\|\bW^{k-1}-\bJ\|_2^2]+(1+\frac{1}{\delta_1})\E[\|\bDelta_x^{k-1} - \overline{\bDelta}_x^{k-1} \|_\normF^2]\nonumber\\
       &\overset{(b)}{\le} (1-p)(1+\delta_1)\lambda^2\E[\|\bPhi_x^{k-1}\|_\normF^2] + (1+\frac{1}{\delta_1})\E[\|\bDelta_x^{k-1} - \overline{\bDelta}_x^{k-1}\|_\normF^2],\label{eq:phi_x_successive}
    \end{align}
    where $(a)$ is based on \eqref{eq:proposition_jessen} and Proposition \ref{propo:matrix_multiply_decompose}, and $(b)$ is from the Assumption \ref{assump:graph}, \ie, $\E[\|\bW^{k-1}-\bJ\|_2^2] = 1-\lambda_p$.
    Further from \eqref{eq:compact-update-x}, we have
    \begin{align}
        \E[\big\|\bDelta_x^{k-1} - \overline{\bDelta}_x^{k-1}\big\|_\normF^2] &= \eta^2\E\left[\big\|\left((T_o+1)\bPhi_y^{k-1}-\sum_{t=1}^{T_o}(\bG^{k,t}-\bG^{k-1})\right)(\bW^{k-1}-\bJ)\big\|_\normF^2\right]\nonumber\\
        &\le (1-p)\lambda^2\eta^2\E[\big\| (T_o+1)\bPhi_y^{k-1}-\sum_{t=1}^{T_o}(\bG^{k,t}-\bG^{k-1})\big\|_\normF^2]\nonumber\\
        &\le 2(1-p)\lambda^2\eta^2\left(\E[\|(T_o+1)\bPhi_y^{k-1}\|_\normF^2]+T_o \sum_{t=1}^{T_o}\E[\|(\bG^{k,t}-\bG^{k-1})\|_\normF^2]\right).\label{eq:delta_x_consensus_error}
    \end{align}
    Substituting \eqref{eq:delta_x_consensus_error} and \eqref{eq:gradient_improvement_local_update} into \eqref{eq:phi_x_successive},
    \begin{align*}
        &\E[\|\bPhi_x^k\|_\normF^2] \nonumber\\
        &\le  (1-p)\lambda^2\left[(1+\delta_1)\E[\|\bPhi_x^{k-1}\|_\normF^2] + 2(1+\frac{1}{\delta_1})\eta^2\left((T_o+1)^2\E[\|\bPhi_y^{k-1}\|_\normF^2 + T_o\sum_{t=1}^{T_o}\|(\bG^{k,t}-\bG^{k-1})\|_\normF^2]\right)\right]\\
        &\le (1-p)\lambda^2\left[\left(1+\delta_1+8(1+\frac{1}{\delta_1})T_o^2L^2\eta^2\right)\E[\|\bPhi_x\|_\normF^2] + 2(1+\frac{1}{\delta_1})(T_o+1)^2\eta^2\E[\|\bPhi_y^{k-1}\|_\normF^2] \right.\\
        &\quad+\left. 8(1+\frac{1}{\delta_1})T_oL^2\eta^2\sum_{t=1}^{T_o}\E[\|\bPhi_x^{k-1,t}\|_\normF^2] + 16(1+\frac{1}{\delta_1})nT_o^2\eta^2\sigma^2/b\right].
    \end{align*}
    By Lemma \ref{lemma:consensus_local_drift}, the consensus error $\E[\|\bPhi_x\|_\normF^2]$ can be upper bounded by
    \begin{align*}
        \E[\|\bPhi_x^k\|_\normF^2] 
        &\le(1-p)\lambda^2\left[\theta_0\E[\|\bPhi_x^{k-1}\|_\normF^2]+2(1+\frac{1}{\delta_1})(T_o+1)^2\eta^2\left(32(T_o+1)^2L^2\eta_l^2 + 1\right)\E[\|\bPhi_y^{k-1}\|_\normF^2]\right.\\
        &\quad\left.+~24(1+\frac{1}{\delta_1})L^2T_o^2(T_o+1)^2\eta^2\eta_l^2\E[\|\barY^{k-1}\|_\normF^2] + 16(1+\frac{1}{\delta_1})nT_o^2\eta^2\left(32T_o^2 L^2\eta_l^2+1\right)\sigma^2/b\right],
    \end{align*}
    where $\theta_0 = 1+\delta_1+80(1+\frac{1}{\delta_1})T_o^2L^2\eta^2$.
    Let $\delta_1 = \frac{\lambda_p}{8}\in (0,1]$. If $\eta\le\frac{\lambda_p}{80(T_o+1)L}$, we have 
    \begin{equation*}
        \theta_0 \le 1+\frac{\lambda_p}{4} \quad \text{and}\quad 1+\frac{1}{\delta_1} < \frac{10}{\lambda_p}.
    \end{equation*}
    Note that $\lambda^2(1 + \frac{1-(1-p)\lambda^2}{4}) < \frac{1+(1+p)\lambda^2}{2}$. Then, 
    \begin{align*}
        \E[\|\bPhi_x^k\|_\normF^2] &< (1-p)\left[\frac{1+(1+p)\lambda^2}{2}\E[\|\bPhi_x^{k-1}\|_\normF^2] + \frac{40\lambda^2}{\lambda_p}(T_o+1)^2\eta^2\E[\|\bPhi_y^{k-1}\|_\normF^2]\right.\\
        &\quad\left.+\frac{240\lambda^2L^2(T_o+1)^4\eta^2\eta_l^2}{\lambda_p}\E[\|\barY^{k-1}\|_\normF^2] + \frac{320\lambda^2n(T_o+1)^2\eta^2}{\lambda_p}\frac{\sigma^2}{b}\right],
    \end{align*}
    since $\eta_l\le\frac{1}{8L(T_o+1)}$.
\end{proof}

%%%%%%%%%%%%%%%%%%%%%%%%%
\begin{lemma}[Tracking error for communication updates]\label{lemma:tracking_error_communication_updates}
Suppose Assumption \ref{assump:graph},\ref{assump:smooth} and \ref{assump:bounded_variance} hold. If $\eta_l\le\frac{1}{8L(T_o+1)}$ and $\eta\le \frac{\lambda_p}{80L(T_o+1)}$, we have
\begin{align*}
    \E[\|\bPhi_y^k\|_{\normF}^2]
    &\le (1-p)\lambda^2\!\left[\frac{1+(1+p)}{2} \E[\|\bPhi_y^{k-1}\|_{\normF}^2]\!+\!\frac{400}{\lambda_p}L^2\E[\|\bPhi_x^{k-1}\|_{\normF}^2] \right]\\
    &\quad+\frac{(1-p)\lambda^2}{\lambda_p}\left[125L^2\eta^2(T_o+1)^2\E[\|\barY^{k-1}\|_{\normF}^2] + 180n \frac{\sigma^2}{b}\right].
\end{align*}
\end{lemma}
\begin{proof}
    Let $\bDelta_y^{k-1} = \bY^{k}-\bY^{k-1}\bW^{k-1}$ and $\overline{\bDelta}_y^{k-1} = \barY^{k}-\barY^{k-1}$. From \eqref{eq:compact-update-y} and Lemma \ref{lemma:y_tracking}, we have 
    \begin{align*}
        \E[\big\|\bDelta_y^{k-1} - \overline{\bDelta}_y^{k-1}\big\|_\normF^2] &= \E[\big\|\left(\bG^{k}-\bG^{k-1}\right)(\bW^{k-1}-\bJ)\big\|_\normF^2]\nonumber\\
        &\le (1-p)\lambda^2\E[\big\|\bG^k - \bG^{k-1} \big\|_\normF^2]\\
        &= (1-p)\lambda^2 \E[\|\bG^{k} - \nabla F(\bX^k) + \nabla F(\bX^k) -\nabla F(\bX^{k-1})-\left(\bG^{k-1} - \nabla F(\bX^{k-1})\right)\|_{\normF}^2]\\
        &\le (1-p)\lambda^2\left(\frac{6n\sigma^2}{b} + 3\E[\|\nabla F(\bX^k) -\nabla F(\bX^{k-1})\|_{\normF}^2]\right)\\
        &\le (1-p)\lambda^2\left(\frac{6n\sigma^2}{b} + 3L^2\E[\|\bX^k-\bX^{k-1}\|_{\normF}^2]\right).
    \end{align*}
    where the first inequality is based on Assumption \ref{assump:graph}, the penultimate inequality is due to Assumption \ref{assump:bounded_variance} and Proposition \ref{propo:sum_of_varaibles_decompose}, and the last one is from Assumption \ref{assump:smooth}.
    Together with Lemma \ref{lemma:y_tracking}, we have
    \begin{align}
    \E[\|\bPhi_y^k\|_\normF^2]
    &= \E[\|(\bY^{k-1} - \barY^{k-1})\left(\bW^{k-1}-\bJ\right) + \bDelta_y^{k-1}-\overline{\bDelta}_y^{k-1}\|_\normF^2]\nonumber\\
    &\le (1-p)\lambda^2(1+\delta_1)\E[\|\bPhi_y^{k-1}\|_{\normF}^2] + (1+\frac{1}{\delta_1}) \E[\|\bDelta_y^{k-1}-\overline{\bDelta}_y^{k-1}\|_{\normF}^2]\nonumber\\
    &\le  (1-p)\lambda^2\left[(1+\delta_1)\E[\|\bPhi_y^{k-1}\|_{\normF}^2] + (1+\frac{1}{\delta_1}) \left(\frac{6n\sigma^2}{b} + 3L^2\E[\|\bX^k-\bX^{k-1}\|_{\normF}^2]\right)\right].\label{eq:phi_y_k}
    \end{align}
    Moreover, from Lemma \ref{lemma:progress_improvement_iterates} and Lemma \ref{lemma:consensus_local_drift}, if $\eta_l \le \frac{1}{8L(T_o+1)}$,
    \begin{align*}
    &\E[\|\bX^k-\bX^{k-1}\|_{\normF}^2]\\
    &\le 6\eta^2(T_o+1)^2\left((1-p)\lambda^2+32T_o(T_o+1)L^2\eta_l^2\right)\E[\|\bPhi_y^{k-1}\|_{\normF}^2] +12\left(1+20T_o^2L^2\eta^2\right)\E[\|\bPhi_x^{k-1}\|_{\normF}^2] \\
    &\quad+ 48nT_o^2\eta^2\left(1+32T_o^2 L^2 \eta_l^2\right)\sigma^2/b + 3\eta^2(T_o+1)^2\left(1+24T_o^2L^2\eta_l^2\right)\E[\|\barY^{k-1}\|_{\normF}^2]
    \end{align*}
    It follows that
    \begin{align*}
       \E[\|\bPhi_y^k\|_\normF^2] &\le (1-p)\lambda^2\left[\theta_1\E[\|\bPhi_y^{k-1}\|_{\normF}^2]+ 36(1+\frac{1}{\delta_1})(1+20T_o^2L^2\eta^2)L^2 \E[\|\bPhi_x^{k-1}\|_{\normF}^2]+6n(1+\frac{1}{\delta_1})\right.\\
       &\quad\left(1+48T_o^2L^2\eta^2(1+32T_o^2L^2\eta_l^2)\right)\frac{\sigma^2}{b} + \left.9(1+\frac{1}{\delta_1})L^2(T_o+1)^2(1+24T_o^2L^2\eta_l^2)\eta^2\E[\|\barY^{k-1}\|_{\normF}^2] \right]
    \end{align*}
    where $\theta_1 = 1+\delta_1+ 36(1+\frac{1}{\delta_1})(T_o+1)^2L^2\eta^2$. If $\eta\le \frac{\lambda_p}{80L(T_o+1)}$ and $\delta_1 = \frac{\lambda_p}{8}\in(0,1)$,
    
    \begin{equation*}
        \theta_1 \le 1+\frac{\lambda_p}{4} \quad \text{and}\quad 1+\frac{1}{\delta_1}\le \frac{10}{\lambda_p}.
    \end{equation*}
    Note that $\lambda^2(1+\frac{1-(1-p)\lambda^2}{4}) \le \frac{1+(1+p)\lambda^2}{2}$. 
    Thus, 
    \begin{align*}
        \E[\|\bPhi_y^k\|_{\normF}^2] &\le (1-p)\left[
            \frac{1+(1+p)\lambda^2}{2} \E[\|\bPhi_y^{k-1}\|_{\normF}^2] +\frac{360\lambda^2}{\lambda_p}\left(1+20T_o^2L^2\eta^2\right)L^2\E[\|\bPhi_x^{k-1}\|_{\normF}^2] \right.\\
        &\quad\left.+\frac{90\lambda^2}{\lambda_p}L^2\eta^2(T_o+1)^2\left(1+24T_o^2L^2\eta_l^2\right)\E[\|\barY^{k-1}\|_{\normF}^2] + \frac{180n\lambda^2}{\lambda_p}\frac{\sigma^2}{b}\right].
    \end{align*}
    Since $\eta_l\le\frac{1}{8L(T_o+1)}$ and $\eta\le\frac{\lambda_p}{80L(T_o+1)}$, we arrive at Lemma \ref{lemma:tracking_error_communication_updates}.

\end{proof}

\subsection{Descent lemma}
\begin{lemma}[Descent lemma]\label{lemma:descent}
    Suppose Assumption \ref{assump:graph}, \ref{assump:smooth} and \ref{assump:bounded_variance} hold. If 
    $\eta_l \le \frac{1}{8L(T_o+1)}$ and $\eta\le\frac{1}{6L(T_o+1)}$, we have
    \begin{align}
            \E [f(\barx^{k+1})] - \E [f(\barx^{k})]\nonumber
            &\le -\frac{\eta(T_o+1)}{4}\E[\|\nabla f(\barx^k)\|_2^2]+\frac{3L^2 (T_o+1)^3\eta\eta_l^2}{n}\E[\|\barY^k\|_{\normF}^2]\nonumber\\
            &\quad + \frac{10(T_o+1)L^2\eta}{n}\left(\E[\|\bPhi_x^k\|_{\normF}^2]+(T_o+1)^2\eta_l^2\E[\|\bPhi_y^k\|_{\normF}^2]\right)\nonumber\\
            &\quad+\left(64L^2T_o^3\eta\eta_l^2+\frac{3LT_o\eta^2}{2n}\right)\frac{\sigma^2}{b},\label{eq:descent_lemma}
    \end{align}
for any $k\ge 0$.
\end{lemma}
\begin{proof}
    Because of the $L$-smooth of the objective $f$,
    \begin{equation}\label{eq:smooth_def}
        \begin{aligned}
        &\E [f(\barx^{k+1})] - \E [f(\barx^{k})]
        & \le \E [ \nabla f(\barx^k)^{\top}(\barx^{k+1}-\barx^k)] + \frac{L}{2n}\E[\|\barX^{k+1}-\barX^{k}\|_{\normF}^2].
        \end{aligned}
    \end{equation}
    We first handle the first term on the right-hand side. Let $\bg^{k+1,t} = \frac{1}{n}\sum_{i=1}^n \bg_i^{k+1,t}$. With Lemma \ref{lemma:y_tracking} and the \eqref{eq:compact-update-x} in hand, we have
    \begin{align}
        &\E [\nabla f(\barx^k)^{\top}(\barx^{k+1}-\barx^k)]\nonumber\\
        &=  -\eta\sum_{t=0}^{T_o} \E[ \nabla f(\barx^k)^{\top} \barg^{k+1,t}]\nonumber\\
        &\overset{(a)}{=} -\eta (T_o+1) \E[\|\nabla f(\barx^k)\|_{2}^2] -\eta \sum_{t=0}^{T_o} \E\left[\nabla f(\barx^k)^{\top}\left(\frac{1}{n}\sum_{i=1}\nabla f_i(\bx_i^{k+1,t})-\nabla f(\barx^k)\right)\right]\nonumber\\
        &\overset{(b)}{\le}-\eta (T_o+1) \E[\| \nabla f(\barx^k)\|_2^2] +\frac{\eta}{2}\sum_{t=0}^{T_o} \E\left[\|\nabla f(\barx^k)\|_2^2+ \|\frac{1}{n}\sum_{i=1}\nabla f_i(\bx_i^{k+1,t})-\nabla f(\barx^k)\|_2^2\right]\nonumber\\
        &\le -\frac{\eta (T_o+1)}{2}\E[\|\nabla f(\barx^k)\|_2^2]+ \frac{\eta}{2}\sum_{t=0}^{T_o} \E\left[\|\frac{1}{n}\sum_{i=1}\nabla f_i(\bx_i^{k+1,t})-\nabla f(\barx^k)\|_2^2 \right]\nonumber\\
        &\overset{(c)}{\le} -\frac{\eta (T_o+1)}{2}\E[\|\nabla f(\barx^k)\|_2^2] + \frac{\eta}{2n}\sum_{t=0}^{T_o}\sum_{i=1}^n \E[\|\nabla f_i(\bx_i^{k+1,t}) - \nabla f_i(\barx^k)\|_2^2]\nonumber\\
        &\overset{(d)}{\le}-\frac{\eta (T_o+1)}{2}\E[\|\nabla f(\barx^k)\|_2^2] + \frac{\eta L^2}{2n}\sum_{t=0}^{T_o} \E[\|\bPhi_x^{k,t}\|_{\normF}^2]\label{eq:smooth_firstterm}.
    \end{align}
    Here, $(a)$ is from Assumption \ref{assump:bounded_variance}, \ie, $\E[\bg_i^{k+1,t}]=\nabla f_i(\bx_i^{k+1,t})$;
     $(b)$ is based on $2\boldsymbol{a}^{\top}\boldsymbol{b}\le \|\boldsymbol{a}\|_2^2+ \|\boldsymbol{b}\|_2^2$ for any vectors $\boldsymbol{a},\boldsymbol{b}\in\mathbb{R}^d$;
    $(c)$ uses the Proposition \ref{propo:sum_of_varaibles_decompose}; and $(d)$ follows Assumption \ref{assump:smooth}. \\
    
    Next, we will bound the second term on the right-hand side of \eqref{eq:smooth_def}. From Lemma \ref{lemma:progress_improvement_average}, we can control the term $\E[\|\barX^{k+1}-\barX^k\|_F^2]$ by
    \begin{equation}\label{eq:smooth_secondterm}
        \E[\|\barX^{k+1}-\barX^k\|_{\normF}^2]
        \le 3\eta^2\left[\frac{T_o\sigma^2}{b}+T_oL^2\sum_{t=0}^{T_o}\E[\|\bPhi_x^{k,t}\|_{\normF}^2]+nT_o^2\E[\|\nabla f(\barx^k)\|_2^2]\right].
    \end{equation}
    
    Substituting \eqref{eq:smooth_firstterm} and \eqref{eq:smooth_secondterm} into \eqref{eq:smooth_def}, we have
    \begin{align*}
        &\E [f(\barx^{k+1})] - \E [f(\barx^{k})]\\
        &\le \left(-\frac{\eta(T_o+1)}{2}+\frac{3LT_o^2\eta^2}{2}\right)\E[\|\nabla f(\barx^k)\|_2^2] + \frac{\eta L^2}{2n}(1 + 3T_oL\eta)\sum_{t=0}^{T_o}\E[\|\bPhi_x^{k,t}\|_{\normF}^2] + \frac{3LT_o\eta^2\sigma^2}{2nb}\\
        &\le-\frac{\eta (T_o+1)}{4}\E[\|\nabla f(\barx^k)\|_2^2] + \frac{\eta L^2}{2n}(1 + 3T_oL\eta)\sum_{t=0}^{T_o}\E[\|\bPhi_x^{k,t}\|_{\normF}^2] + \frac{3LT_o\eta^2\sigma^2}{2nb}.
    \end{align*}
    where the last inequality holds if $\eta\le \frac{1}{6L(T_o+1)}$.
    Recall Lemma  \ref{lemma:consensus_local_drift} that if $\eta_l \le \frac{1}{8L(T_o+1)}$, we have
    \begin{equation*}
        \sum_{t=0}^{T_o}\E[\|\bPhi_x^{k,t}\|_{\normF}^2]\le 10T_o\E[\|\bPhi_x^k\|_{\normF}^2]+ 8\eta_l^2(T_o+1)^3\E[\|\bPhi_y^k\|_{\normF}^2]+\frac{64n\eta_l^2 T_o^3 \sigma^2}{b} + 3\eta_l^2  (T_o+1)^3\E[\|\barY^k\|_{\normF}^2].
    \end{equation*}
    
    Therefore,
    \begin{align*}
        &\E [f(\barx^{k+1})] - \E [f(\barx^{k})]\\
        &\le -\frac{\eta(T_o+1)}{4}\E[\|\nabla f(\barx^k)\|_2^2] + \frac{5 T_oL^2\eta}{n}(1 + 3T_oL\eta)\E[\|\bPhi_x^k\|_{\normF}^2] +  \frac{4 (T_o+1)^3L^2\eta\eta_l^2}{n}(1 + 3T_oL\eta)\E[\|\bPhi_y^k\|_{\normF}^2]\\
        &\quad+\frac{64n(1+3T_oL\eta)L^2T_o^3\eta\eta_l^2+3LT_o\eta^2}{2nb}\sigma^2 + \frac{3L^2 (T_o+1)^3\eta\eta_l^2}{2n}(1 + 3T_oL\eta)\E[\|\barY^k\|_{\normF}^2]\\
        &\le -\frac{\eta(T_o+1)}{4}\E[\|\nabla f(\barx^k)\|_2^2] + \frac{10(T_o+1)L^2\eta}{n}\left(\E[\|\bPhi_x^k\|_{\normF}^2]+(T_o+1)^2\eta_l^2\E[\|\bPhi_y^k\|_{\normF}^2]\right)\\
        &\quad+\left(64L^2T_o^3\eta\eta_l^2+\frac{3LT_o\eta^2}{2n}\right)\frac{\sigma^2}{b}+\frac{3L^2 (T_o+1)^3\eta\eta_l^2}{n}\E[\|\barY^k\|_{\normF}^2].
    \end{align*}
\end{proof}

\section{Proof of Theorem \ref{thm:stochastic}}\label{appendix:proof_thm1}

In this section, we present the convergence analysis of \alg.
From the descent lemma, \ie, Lemma \ref{lemma:descent}, summing \eqref{eq:descent_lemma} from $k=0$ to $k=K-1$ gives
\begin{align}
   \frac{\eta (T_o+1)}{4}\sum_{k=0}^{K-1}\E[\|\nabla f(\barx^k)\|_2^2]
   &\le \tilde{f}+ \frac{3L^2(T_o+1)^3\eta\eta_l^2}{n}\sum_{k=0}^{K-1}\E[\|\barY^k\|_{\normF}^2]\nonumber\\
   & +\frac{10(T_o+1)L^2\eta }{n}\sum_{k=0}^{K-1}\left( \E[\|\bPhi_x^k\|_{\normF}^2] + (T_o+1)^2\eta_l^2\E[\|\bPhi_y^k\|_{\normF}^2]\right)\nonumber\\
   & +\left(64L^2T_o^3\eta\eta_l^2+\frac{3LT_o\eta^2}{2n}\right)\frac{K\sigma^2}{b}, \label{eq:sum_descent_lemma}
\end{align}
where $\tilde{f} = f(\barx^0) - f^*$.

To show the convergence of $\frac{1}{K}\sum_{k=0}^{K-1}\E[\|\nabla f(\barx^k)\|_2^2]$, we need to bound right hand side of \eqref{eq:sum_descent_lemma}.
%%%%%%%%%%%%%%%%%%%%%%%%%%%%%%%%%%%%%%%%%%%%%%%%%%%
To this end, we first formulate the dynamics of consensus errors and tracking errors. For any $k>0$, let 
\begin{equation*}
    \bPhi^k \triangleq  \begin{bmatrix}\E[\|\bPhi_x^k\|_{\normF}^2]\\\E[\|\bPhi_y^k\|_{\normF}^2]\end{bmatrix}
    \qquad\text{and}\qquad
    \be^k \triangleq \begin{bmatrix}\E[\|\barY^k\|_{\normF}^2]\\ \frac{\sigma^2}{b}\end{bmatrix}.
\end{equation*}
Assuming that $\eta \le \frac{\lambda_p}{80L(T_o+1)}$ and $\eta_l\le\frac{1}{8L(T_o+1)}$, we can formulate the dynamics from Lemma \ref{lemma:consensus_error_communication_updates} and Lemma \ref{lemma:tracking_error_communication_updates},
$$\bPhi^{k+1} \le (1-p)\bA \bPhi^k + \frac{(1-p)\lambda^2}{\lambda_p}\bB\be^k,$$
where 
\begin{align}
    \bA &= \begin{bmatrix}
        \dfrac{1+(1+p)\lambda^2}{2} & \dfrac{40\lambda^2(T_o+1)^2\eta^2}{\lambda_p}\\
        \dfrac{400\lambda^2L^2}{\lambda_p} &  \dfrac{1+(1+p)\lambda^2}{2} 
    \end{bmatrix}\label{eq:A},\\[7pt]
    \bB &= \begin{bmatrix}
        240L^2(T_o+1)^4\eta^2\eta_l^2 & 320n(T_o+1)^2\eta^2\\
        125L^2(T_o+1)^2\eta^2 & 180n
    \end{bmatrix}.\label{eq:B}
\end{align}
By telescoping, we have
\begin{equation}\label{eq:phi_unrolling}
    \bPhi^{k}\le (1-p)^k\bA^k\bPhi^0 + \!\sum_{t=0}^{k-1}\!\left((1-p)\bA\right)^t\!\frac{(1-p)\lambda^2}{\lambda_p}\bB\be^{k-1-t}.
\end{equation}
Summing \eqref{eq:phi_unrolling} from $k=0$ to $K$ gives 
\begin{align}
    \sum_{k=0}^K \bPhi^k &\le \sum_{k=0}^K (1-p)^k\bA^k\bPhi^0 + \sum_{k=0}^K \sum_{t=0}^{k-1}\left((1-p)\bA\right)^t\frac{(1-p)\lambda^2}{\lambda_p}\bB\be^{k-1-t}\nonumber\\
    &\le \left(\sum_{k=0}^{\infty} (1-p)^k\bA^k\right)\bPhi^0 + \left(\sum_{k=0}^{\infty} \left((1-p)\bA\right)^k\right)\sum_{k=0}^{K-1}\frac{(1-p)\lambda^2}{\lambda_p}\bB\be^{k},\label{eq:phi_unrolling_sum}
\end{align}
where we define $0^0 = 1$.

%%%%%%%%%%%%%%%%%%%%%%%%%%%%%%%%%%%%%%%%%%%%%%%%%%%
To control  $\sum_{k=0}^{\infty} (1-p)^k\bA^k$ in \eqref{eq:phi_unrolling_sum}, we then establish the following lemma   implying that $\bI-(1-p)\bA$ is invertible.

\begin{lemma}[The spectral radius of $\bA$]\label{lemma:spectral_radius_A} 
    If $\eta\le\frac{(1+p)\lambda_p^2}{80\sqrt{10}(T_o+1)L}$,
$$\rho\left((1-p)\bA\right)< 1,$$
where $\rho(\bA)$ denotes the spectral radius of $\bA$ defined in \eqref{eq:A}.
\end{lemma}
\begin{proof}
    
    The eigenvalue $x$ of $\bA$ satisfies
    \begin{equation}\label{eq:poly}
        p(x) \coloneqq x^2 - (1+(1+p)\lambda^2)x + \left(\frac{(1+(1+p)\lambda^2)^2}{4} - \frac{16000\lambda^4(T_o+1)^2L^2\eta^2}{\lambda_p^2}\right)=0.
    \end{equation}
 Suppose that $x_1$ and $x_2$ are the two solutions to \eqref{eq:poly}.  Without the loss of generality, we have $x_1\ge x_2$ and
 \begin{align*}
     x_1 &= \frac{1+(1+p)\lambda^2}{2}+ \frac{40\sqrt{10}\lambda^2(T_o+1)L\eta}{\lambda_p},\\
     x_2 &=  \frac{1+(1+p)\lambda^2}{2} - \frac{40\sqrt{10}\lambda^2(T_o+1)L\eta}{\lambda_p}.
 \end{align*}
 If $\eta\le\frac{(1+p)\lambda_p^2}{80\sqrt{10}(T_o+1)L}$, we have 
 \begin{equation*}
    \frac{40\sqrt{10}\lambda^2(T_o+1)L\eta}{\lambda_p}\le \frac{(1+p)\lambda_p\lambda^2}{2}.
\end{equation*}
Correspondingly,
\begin{equation*}
    \rho(\bA) = \max\{|x_1|,|x_2|\} = |x_1| \le \frac{1+(1+p)\lambda^2(1+\lambda_p)}{2}
\end{equation*}
and 
\begin{equation*}
    \rho((1-p)\bA)\le \frac{1-p + (1+p)(1-\lambda_p^2)}{2} < 1,
\end{equation*}
where the last line is based on Assumption \ref{assump:graph}.
\end{proof}

%%%%%%%%%%%%%%%%%%%%%%%%%%%%%%%%%%
Due to the invertibility of $\bI-(1-p)\bA$, it follows \cite[Corollary 5.6.16]{horn2012matrix} that
\begin{equation*}
    \sum_{k=0}^{\infty} (1-p)^k\bA^k = \left(\bI - (1-p)\bA\right)^{-1},
\end{equation*}
such that \eqref{eq:phi_unrolling_sum} becomes
\begin{equation}\label{eq:sum_of_consensus_error}
    \sum_{k=0}^K \bPhi^k \le \left(\bI-(1-p)\bA\right)^{-1}\bPhi^0 + \bC\sum_{k=0}^{K-1}\be^{k},
\end{equation}
where
\begin{equation*}
    \bC = \left(\bI-(1-p)\bA\right)^{-1}\frac{(1-p)\lambda^2}{\lambda_p}\bB.
\end{equation*}

Now, we are going to control the upper bound of the $\left(\bI-(1-p)\bA\right)^{-1}$, \ie, the upper bound of 
\begin{equation*}
     \frac{1}{\det(\bI-(1-p)\bA)}\operatorname{adj}(\bI-(1-p)\bA),
\end{equation*}
where $\det(\bA)$ means the determinant of $\bA$ and $\operatorname{adj}(\bA)$ represents the adjugate of $\bA$.
If the step-size further satisfies $\eta \le \frac{(1+p)\lambda_p^2}{360(T_o+1)L}$, we have
\begin{align*}
    \det(\bI-(1-p)\bA)
    & = \left(\frac{(1+p)\lambda_p}{2}\right)^2 -  \frac{16000\lambda^4(T_o+1)^2L^2(1-p)^2\eta^2}{\lambda_p^2}\\
    &\ge \frac{(1+p)^2\lambda_p^2}{4}  - \frac{(1+p)^2\lambda_p^2}{8} = \frac{(1+p)^2\lambda_p^2}{8}.
\end{align*}
Then,
\begin{align*}
        (\bI-(1-p)\bA)^{-1}
        &\le \frac{4}{(1+p)^2\lambda_p^3} \begin{bmatrix}
            (1+p)\lambda_p^2 & 80(1-\lambda_p)(T_o+1)^2\eta^2\\
             800(1-\lambda_p)L^2  & (1+p)\lambda_p^2
        \end{bmatrix},
    \end{align*}
and 
\begin{align*}
    \bC \!&\le \!\frac{240(1-\lambda_p)}{(1+p)^2\lambda_p^4}\begin{bmatrix}
        4c_{1} L^2(T_o+1)^4\eta^2\eta_l^2 &6nc_{2}(T_o+1)^2\eta^2\\
        3c_{3} L^2(T_o+1)^2\eta^2& 3nc_{4}
    \end{bmatrix},
\end{align*}
where
\begin{align*}
    c_{1} &= \left((1+p)\lambda_p^2 + 60(1-\lambda_p)\eta_c^2\right),\\
    c_{2} &= \left((1+p)\lambda_p^2 + 40(1-\lambda_p)\right),\\
    c_{3} &= \left((1+p)\lambda_p^2 + 1100(1-\lambda_p)L^2(T_o+1)^2\eta_l^2\right),\\
    c_{4} &= \left((1+p)\lambda_p^2  + 1600(1-\lambda_p)L^2(T_o+1)^2\eta^2 \right).
\end{align*}

Thus, if $\bX^0 = \bx^0\boldsymbol{1}_n^{\top}$,
\begin{subequations}\label{eq:sum_phi_case2}
\begin{equation}\label{eq:sum_phi_x_case2}
\begin{aligned}
    \sum_{k=0}^{K-1} \E[\|\bPhi_x^k\|_{\normF}^2]\le\sum_{k=0}^K \E[\|\bPhi_x^k\|_{\normF}^2]
    &\le \frac{320(1-\lambda_p)(T_o+1)^2\eta^2}{(1+p)^2\lambda_p^3}\E[\|\bPhi_y^{0}\|_{\normF}^2] \\
    &\quad+\frac{960c_{1}(1-\lambda_p) L^2(T_o+1)^4\eta^2\eta_l^2}{(1+p)^2\lambda_p^4}\sum_{k=0}^{K-1}\E[\|\barY^k\|_{\normF}^2] \\
    &\quad+ \frac{1440c_{2}(1-\lambda_p)(T_o+1)^2\eta^2}{b(1+p)^2\lambda_p^4}nK\sigma^2,
\end{aligned}
\end{equation}
and
\begin{equation}\label{eq:sum_phi_y_case2}
\begin{aligned}
    \sum_{k=0}^{K-1} \E[\|\bPhi_y^k\|_{\normF}^2]\le\sum_{k=0}^K \E[\|\bPhi_y^k\|_{\normF}^2]
    &\le \frac{4}{(1+p)\lambda_p}\E[\|\bPhi_y^{0}\|_{\normF}^2] \\
    &\quad+ \frac{720c_{3}(1-\lambda_p)L^2(T_o+1)^2\eta^2}{(1+p)^2\lambda_p^4}\sum_{k=0}^{K-1}\E[\|\barY^k\|_{\normF}^2] \\
    &\quad+ \frac{720 c_{4}(1-\lambda_p)}{b(1+p)^2\lambda_p^4}nK\sigma^2.
\end{aligned}
\end{equation}
\end{subequations}

%%%%%%%%%%%%%%%%%%%%%%%%%%%%%%%%%%%%%%%%%%
Substituting \eqref{eq:sum_phi_case2} into the \eqref{eq:sum_descent_lemma}, we have
\begin{equation}\label{eq:finish_up1}
\begin{aligned}
   \frac{\eta (T_o+1)}{4}\sum_{k=0}^{K-1}\E[\|\nabla f(\barx^k)\|_2^2]
   & \le \tilde{f}+ \underset{T_1}{\underbrace{\frac{15L^2(T_o+1)^3\eta\eta_l^2}{n}\sum_{k=0}^{K-1} \E[\|\barY^k\|_{\normF}^2]}} +\underset{T_2}{\underbrace{\frac{80c_{1}L^2(T_o+1)^3\eta\eta_l^2}{n(1+p)^2\lambda_p^3}\E[\|\bPhi_y^0\|_{\normF}^2]}}\\
   & \quad+ \left(T_3+\frac{3LT_o\eta^2}{2n}\right)K\sigma^2/b
\end{aligned}
\end{equation}
where $T_3 =\frac{7200(84\eta_c^2+2(1+p)\lambda_p^2)(1-\lambda_p)L^2(T_o+1)^3\eta\eta_l^2}{(1+p)^2\lambda_p^4}+64L^2(T_o+1)^3\eta\eta_l^2$.
Together with \eqref{eq:barY} and \eqref{eq:sum_phi_x_case2}, 
we have
\begin{align*}
    \sum_{k=0}^{K-1} \E[\|\barY^k\|_{\normF}^2] 
    &\le 3K\sigma^2/b + 3L^2 \sum_{k=0}^{K-1}\E[\|\bPhi_x^k\|_{\normF}^2]+3n\sum_{k=0}^{K-1}\E[\|\nabla f(\barx^k)\|_2^2]\\
    &\le 6n\sum_{k=0}^{K-1}\E[\|\nabla f(\barx^k)\|_2^2] + \frac{1920(1-\lambda_p)L^2(T_o+1)^2\eta^2}{(1+p)^2\lambda_p^3}\E[\|\bPhi_y^{0}\|_{\normF}^2]\\
    &\quad+ \left(\frac{1440c_{2}(1-\lambda_p) L^2(T_o+1)^2\eta^2}{(1+p)^2\lambda_p^4}+\frac{1}{n}\right)6nK\sigma^2/b
\end{align*}
where the last inequality is due to the step-size conditions, \ie, $\eta\le\frac{(1+p)\lambda_p^2}{360L(T_o+1)}$ and $\eta_l\le \frac{1}{8L(T_o+1)}$ s.t. 
\begin{equation*}
    \frac{3\cdot 960 c_{1}L^4(T_o+1)^4\eta^2\eta_l^2}{(1+p)^2\lambda_p^4}\le \frac{1}{2}.
\end{equation*}
By further assuming that $\eta_l\le\frac{1}{27L(T_o+1)}$ we have 
\begin{align*}
   T_1 &\le  \frac{\eta(T_o+1)}{8}\sum_{k=0}^{K-1}\E[\|\nabla f(\barx^k)\|_2^2]+\frac{\alpha_0(1-p)(T_o+1)\eta}{n}\E[\|\bPhi_y^{0}\|_{\normF}^2]\!+\!\alpha_1 L^2(T_o+1)^3\eta\eta_l^2K\frac{\sigma^2}{b},
\end{align*}
for some positive absolute constant $\alpha_0,\alpha_1$. Suppose that $\eta_c=\alpha\sqrt{1+p}\lambda_p$ and $\eta_l\le \frac{\sqrt{1+p}\lambda_p}{360\alpha L(T_o+1)}$ for some positive $\alpha>0.1$. Then,
\begin{align*}
    &T_2\le \alpha_2\left(\frac{(T_o+1)\eta}{n}\right)\E[\|\bPhi_y^0\|_{\normF}^2],\\
    &T_3 \le \alpha_3\left(\frac{(1-p)L^2(T_o+1)^3\eta^3}{(1+p)^2\lambda_p^4} +L^2(T_o+1)^3\eta\eta_l^2\right)
\end{align*}
for some positive absolute constants $\alpha_2$ and $\alpha_3$.
Therefore, 
\begin{align*}
    &\frac{\eta (T_o+1)}{8}\sum_{k=0}^{K-1}\E[\|\nabla f(\barx^k)\|_2^2]\\
    &\le  \tilde{f}+\!O\!\left(\!\frac{(T_o+1)\eta}{n} \!\right)\!\E[\|\bPhi_y^0\|_{\normF}^2] \!+\!O\!\left(\!\frac{(1-p)L^2(T_o+1)^2\eta^2}{(1+p)^2\lambda_p^4}\right.+\left. L^2(T_o+1)^2\eta_l^2+\frac{3L\eta}{2n}\right)\frac{K(T_o+1)\eta\sigma^2}{b},
 \end{align*}
 which is equivalent to \eqref{eq:thm1_rate}.

%%%%%%%%%%%%%%%%%%%%%

\section{Proof of Corollary \ref{corollary:communication_complexity_minibatch}}\label{appendix:proof_corollary}
By rearranging and relaxing \eqref{eq:thm1_rate}, we have 
\begin{align} 
        \frac{1}{K}\sum_{k=0}^{K-1}\E[\|\nabla f(\barx^k)\|_2^2]
        &\le O\left(\frac{\tilde{f}}{\eta_l \eta_c T_o K} + \frac{\eta_l\eta_cL\sigma^2}{nb}+ \frac{1}{(1+p)\lambda_p^2}\frac{L^2T_o^2\eta_l^2\sigma^2}{b}\right.\left.+ \frac{1}{nK}\E[\|\bPhi_y^0\|_{\normF}^2] \right)\label{eq:simple_rate}.
\end{align}
To further fine-tune the step-size and obtain the exact convergence rate, we establish the following lemma, which is slightly different from the Lemma 17 in \cite{koloskova2020unified}.

\begin{lemma}\label{lemma:fine-tune}
   For any parameters $r_0\ge 0,  a_1> 0, a_2>0$, if $K$ is sufficiently large s.t. $\eta^\prime =  \min \left \{\left(\frac{r_0}{a_1 K}\right)^{\frac{1}{2}},  \left(\frac{r_0}{a_2 K}\right)^{\frac{1}{3}} \right \}\le \bar{\eta}$, \ie,
$$
       K\ge \max \left \{\frac{r_0}{a_1\bar{\eta}^2},\frac{r_0}{a_2\bar{\eta}^3}\right \}, $$
   we have 
   \begin{align*}
       \Psi^K &= \frac{r_0}{\eta^\prime K} + a_1 \eta^\prime + a_2 (\eta^\prime)^2 \le 2\left(\frac{a_1r_0}{K} \right)^{\frac{1}{2}} + 2\left(\frac{\sqrt{a_2}r_0}{K} \right)^{\frac{2}{3}}. 
   \end{align*}
\end{lemma}
\begin{proof}
   We mainly follow the proof of \cite[Lemma 17]{koloskova2020unified}.
   \begin{itemize}
       \item If $\eta^\prime = \left(\frac{r_0}{a_1 K}\right)^{\frac{1}{2}}\le  \left(\frac{r_0}{a_2 K}\right)^{\frac{1}{3}}$, 
           \begin{align*}
               \Psi^K &\le 2\bigg(\frac{r_0a_1}{K}\bigg)^{\frac{1}{2}} + a_2\bigg(\frac{r_0}{a_1 K}\bigg)\le 2\bigg(\frac{r_0a_1}{K}\bigg)^{\frac{1}{2}}  + \bigg(\frac{\sqrt{a_2}r_0}{K} \bigg)^{\frac{2}{3}}.
           \end{align*}
       \item  If $\eta^\prime = \left(\frac{r_0}{a_2 K}\right)^{\frac{1}{3}}\le \left(\frac{r_0}{a_1 K}\right)^{\frac{1}{2}}$,
       \begin{align*}
           \Psi^K &\le 2\left(\frac{\sqrt{a_2}r_0}{K} \right)^{\frac{2}{3}} + a_1\left(\frac{r_0}{a_2 K}\right)^{\frac{1}{3}} \le  2\left(\frac{\sqrt{a_2}r_0}{K} \right)^{\frac{2}{3}}  + \bigg(\frac{r_0a_1}{K}\bigg)^{\frac{1}{2}}.
       \end{align*}
       From $K\ge \max \left \{\frac{r_0}{a_1\bar{\eta}^2},\frac{r_0}{a_2\bar{\eta}^3}\right \}$, we have $\eta^\prime\le\bar{\eta}$.
   \end{itemize}
\end{proof}
Then,  applying Lemma \ref{lemma:fine-tune} with $\eta' = \alpha T_o \eta_l$ and 
\begin{align*}
    r_0 = \frac{\tilde{f}}{\sqrt{1+p}\lambda_p}, \qquad a_1 = \frac{\sqrt{1+p}\lambda_pL\sigma^2}{nT_ob},\qquad
     a_2 = \frac{L^2\sigma^2}{(1+p)\lambda_p^2b},
\end{align*}
we bound the right hand side of \eqref{eq:simple_rate} as
\begin{align*}
    O\left(\bigg(\frac{L\tilde{f}\sigma^2}{nT_obK}\bigg)^{\frac{1}{2}} + \bigg(\frac{L\tilde{f}\sigma}{(1+p)\lambda_p^2\sqrt{b}K}\bigg)^{\frac{2}{3}}+
    \frac{1}{nK}\E[\|\bPhi_y^0\|_{\normF}^2] \right),
\end{align*}
if the number of communication rounds $K$ is sufficiently large, i.e., 
$K\ge\max\left\{\frac{360^2 nT_o bL\tilde{f}}{(1+p)^2\lambda_p^4\sigma^2}, \frac{360^3 bL\tilde{f}}{(1+p)\lambda_p^2\sigma^2} \right\}$
such that $\min \Big\{\left(\frac{r_0}{a_1 K}\right)^{\frac{1}{2}}, \left(\frac{r_0}{a_2 K}\right)^{\frac{1}{3}}\Big\}\le \bar{\eta} = \frac{\sqrt{1+p}\lambda_p}{360L}$.

\end{document}